\documentclass{article}
\usepackage{microtype, graphicx, subfigure, booktabs, hyperref, amsmath, amssymb, mathtools, amsthm, array, multirow, caption, enumerate, chemfig, bbm, makecell, mathrsfs}
\usepackage[capitalize,noabbrev]{cleveref}

\newcommand{\normal}{\mathscr{N}}
\newcommand{\uniform}{\mathscr{U}}
\newcommand{\ligand}{x}
\newcommand{\pocket}{\xi}
\newcommand{\Ligand}{X}
\newcommand{\Pocket}{\Xi}
\newcommand{\noisyligand}{y}
\newcommand{\noisyLigand}{Y}
\newcommand{\hatligand}{\hat{x}}
\newcommand{\dligand}{d_{x}}
\newcommand{\dpocket}{d_\xi}
\newcommand{\cligand}{c_x}
\newcommand{\cpocket}{c_\xi}
\newcommand{\encoder}{$E_{\rm lig}$~}
\newcommand{\decoder}{$E_{\rm poc}$~}

\newtheorem{proposition}{Proposition}
\newcommand{\eg}{\emph{e.g.}}
\newcommand{\ie}{\emph{i.e.}}
\newcommand{\modelname}{VoxBind}  

\newlength\savewidth\newcommand\shline{\noalign{\global\savewidth\arrayrulewidth
  \global\arrayrulewidth 1pt}\hline\noalign{\global\arrayrulewidth\savewidth}}

\newcolumntype{x}[1]{>{\centering\arraybackslash}p{#1pt}}
\newcolumntype{y}[1]{>{\raggedright\arraybackslash}p{#1pt}}
\newcolumntype{z}[1]{>{\raggedleft\arraybackslash}p{#1pt}}
\newcommand{\app}{\raise.17ex\hbox{$\scriptstyle\sim$}}

\usepackage[accepted]{icml2024}

\icmltitlerunning{Structure-based drug design by denoising voxel grids}
\begin{document}

\twocolumn[
\icmltitle{Structure-based drug design by denoising voxel grids}
\icmlsetsymbol{equal}{*}

\begin{icmlauthorlist}
\icmlauthor{Pedro O. Pinheiro}{pd}
\icmlauthor{Arian Jamasb}{pd}
\icmlauthor{Omar Mahmood}{pd}
\icmlauthor{Vishnu Sresht}{pd}
\icmlauthor{Saeed Saremi}{pd}
\end{icmlauthorlist}

\icmlaffiliation{pd}{Prescient Design, Genentech}
\icmlcorrespondingauthor{Pedro O. Pinheiro}{pedro@opinheiro.com}
\icmlcorrespondingauthor{Saeed Saremi}{saremis@gene.com}
\icmlkeywords{generative models, structure-based drug design}
\vskip 0.3in
]
\printAffiliationsAndNotice{} 

\begin{abstract} 
We present \modelname, a new score-based generative model for 3D molecules conditioned on protein structures.~Our approach represents molecules as 3D atomic density grids and leverages a 3D voxel-denoising network for learning and generation.
We extend the neural empirical Bayes formalism~\citep{saremi2019neural} to the conditional setting and generate structure-conditioned molecules with a two-step procedure:
(i) sample noisy molecules from the Gaussian-smoothed conditional distribution with underdamped Langevin MCMC using the learned score function and (ii) estimate clean molecules from the noisy samples with single-step denoising. 
Compared to the current state of the art, our model is simpler to train, significantly faster to sample from, and achieves better results on extensive \emph{in silico} benchmarks---the generated molecules are more diverse, exhibit fewer steric clashes, and bind with higher affinity to protein pockets.
The code is available at \url{https://github.com/genentech/voxbind/}.
\end{abstract}

\section{Introduction}\label{sec:introduction}
The goal of \emph{structure-based drug design} (SBDD)~\citep{blundell1996structure,anderson2003process} is to generate molecules that bind with high affinity to specific 3D structures of target biomolecules.
Traditional computational approaches, such as virtual screening, search over a library of molecules and score them to identify the best binders for a given target of interest~\citep{lyu2019ultra}. However, random search is arguably very inefficient as the chemical space grows exponentially with molecular size~\citep{fink2005virtual}.

Recently, many generative modeling methods have been proposed as alternatives to search-based SBDD (see~\citet{thomas2023integrating} for a review). 
In this setting, the goal is to develop data-driven approaches that generate molecules (\ie, \emph{ligands}) \emph{conditioned} on 3D protein binding sites (\ie, \emph{pockets}). 
Generative models have the promise of exploring the chemical space more efficiently and effectively than search-based approaches. 

SBDD generative models typically represent molecules as discrete voxel grids or atomic point clouds. Voxel-based approaches~\citep{wang2022pocket,ragoza2022generating,wang2022relation} represent atoms (or electron densities) as continuous densities and molecules as a discretization of 3D space into voxel grids (a voxel is a discrete unit of volume). Point-cloud based approaches (\eg,~\citet{luo20213d,schneuing2022structure}) treat atoms as points in 3D Euclidean space and rely on graph neural network (GNN) architectures. 
Current state of the art methods in data-driven SBDD~\citep{guan20233tdiff,guan23ddiff} operate on point clouds and are based on E(3) equivariant diffusion models~\citep{hoogeboom22edm} conditioned on protein pockets: they sample points from a Gaussian prior and iteratively apply the learned reverse conditional diffusion process (over continuous coordinates and discrete atom types and bonds) to generate molecules.

\begin{figure}[!t]
\begin{center}
\centerline{\includegraphics[width=\linewidth]{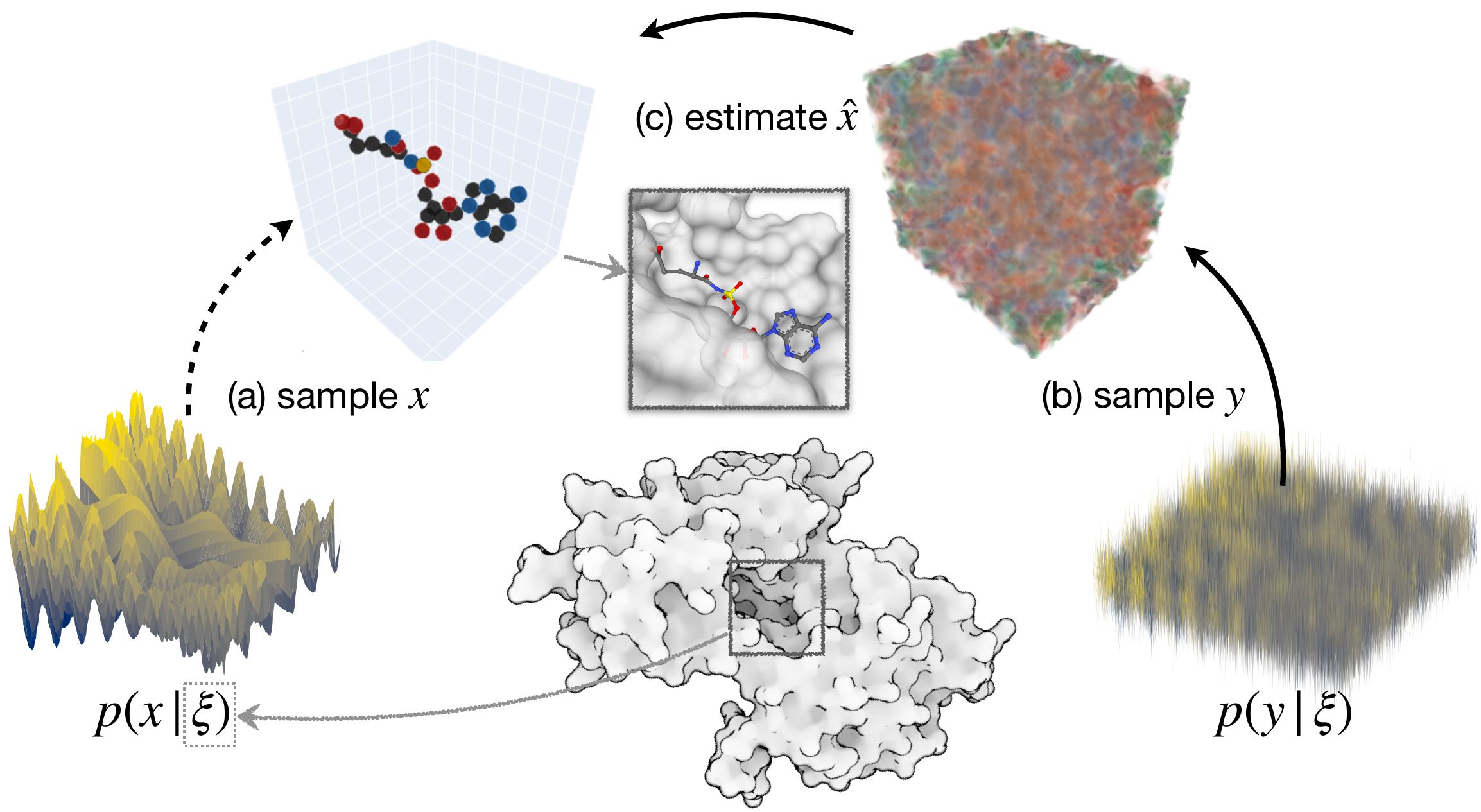}}
\vskip -0.1in
\caption{We are interested in sampling from $p(\ligand|\pocket)$, the distribution of ligands given pocket $\pocket$. 
This is challenging due to the high-dimensionality of the data. Therefore, instead of (a) sampling directly from this distribution, we generate ligands in a two-step procedure: (b) sample $\noisyligand$ from the Gaussian-smoothed distribution $p(\noisyligand|\pocket)$ and (c) estimate the ligand $\hatligand$ from $\noisyligand$ and $\pocket$.}
\label{fig:neb}
\end{center}
\vskip -0.2in
\end{figure}

There is a clear trade-off between the two data representation choices.
On one hand, GNNs can leverage SE(3)-equivariance inductive bias more easily than architectures that operate on voxels~\citep{geiger2022e3nn,weiler20183d}.
On the other hand, they are known to be less expressive due to the message passing formalism~\citep{xu2018powerful,morris2019weisfeiler,pozdnyakov2022incompleteness}. 
The expressivity of these models can be improved with higher-order message passing schemes~\citep{batatia2022mace}. However, they require additional computational cost and have not been applied to 3D generative models yet.
Recently, it has been shown empirically that non-equivariant---but more expressive---models are competitive with equivariant models on different domains, from computer vision~\cite{bai2021transformers,gruver2022lie} to molecule generation~\citep{flam2023language,pinheiro2023voxmol,wang2023generating}.
In fact, equivariance can be learned from large data and strong data augmentations.
Inspired by these findings, we propose a model for SBDD that prioritizes expressivity over SE(3)-equivariance inductive bias. 

Our model, \emph{\modelname}, is a new voxel-based method for 3D ligand generation conditioned on pocket structures. The proposed model generates molecules by extending the \emph{neural empirical Bayes} (NEB) framework~\citep{saremi2019neural, pinheiro2023voxmol} to the structure-conditional setting. 
Given a protein pocket $\pocket$, instead of sampling ligands $\ligand$ directly from $p(\ligand|\pocket)$, we follow a two-step procedure: (i) sample noisy molecules $y$ from the Gaussian-smoothed distribution $p(\noisyligand|\pocket)$ and (ii) estimate the clean ligand from $\noisyligand$ and $\pocket$. \autoref{fig:neb} illustrates our approach.
Sampling is done with Langevin Markov chain Monte Carlo (MCMC), which is known to work much better on the smoother distribution than on the original~\citep{saremi2019neural}. 
We train a conditional denoiser---a model that predicts a clean ligand, given a noisy version of it and its binding pocket---to approximate the conditional score function of the smoothed distribution and the ligand estimator, necessary for steps (i) and (ii), respectively. 

\modelname~is fundamentally different than current state of the art, \ie, diffusion models on point clouds.
First, voxelized representations allow us to use the same type of denoising architectures used in score-based generative models on images. These neural networks are very flexible, scalable and work well in many conditional generation applications~\cite{rombach2022high, saharia2022palette, saharia2205photorealistic}.
Second, convolutional filters (and self-attention layers applied on regular grids) can arguably capture 3D patterns, surfaces and shape complementarity---useful features for structure conditioning---more effectively than message passing (see the many-body representation hypothesis discussed by~\citet{townshend2020atom3d}).
Third, contrary to diffusion models, the noise process used in our approach does not displace atoms in space. This provides a natural way to avoid clashes between generated ligands and pockets.
Finally, our approach only requires a single (fixed) noise level, making training and sampling considerably simpler.

These differences are reflected on empirical results: \modelname~generates better ligands (conditioned on protein binding sites) on CrossDocked2020~\citep{francoeur2020three}, a standard dataset for this task. 
The molecules generated by our model bind with higher affinity to protein pockets (computed with standard docking score software), are more diverse, exhibit fewer steric clashes and lower strain energies. 

Our contributions can be summarized as follows. We propose \modelname: a new score-based generative model for structure-based drug design. 
The proposed method is new, simple and effective. We show that \modelname~outperforms current state of the art on an extensive number of metrics, while being significantly faster to generate samples.

\section{Related work}\label{sec:related_works}

\subsection{Unconditional 3D molecule generation}
Many recent 3D molecule generation methods represent atoms as points (with 3D coordinates, atom types and possibly other features) and molecules as a set of points.
For instance,~\citet{gebauer2018generating,gebauer2019symmetry,luo2022autoregressive} propose autoregressive approaches to generation, where points (atoms) are sampled iteratively over time, while~\citet{kohler2020equivariant,satorras2021n} generate molecules using normalizing flows~\citep{rezend2015normflow}.
\citet{hoogeboom22edm} propose the equivariant diffusion model (EDM), a diffusion-based~\citep{sohl2015diff} generative model that operates on point clouds. EDMs learn to denoise a diffusion process (operating on both continuous and categorical data) and generate molecules by iteratively applying the denoising network on an initial noise. Many follow-up work propose extensions to EDM~\citep{huang2022mdm,vignac2023midi,xu2023geometric}.

\citet{ragoza2020vox} propose a different way to generate molecules: map atomic densities to discrete voxel grids and leverage computer vision techniques for generation. In particular, they propose a generation method based on variational autoencoders~\citep{kingma2014vae}.
More recently,~\citet{pinheiro2023voxmol} propose VoxMol, a score-based generative model that operates on voxelized grids. They show that voxel-based representations achieve results comparable to point-cloud representations on unconditional molecule generation. Our work can be seen as an extension of VoxMol (and NEB) to the conditional generation setting.

\subsection{Pocket-conditional molecule generation}
The first methods proposed for this task represent molecules as 3D voxel grids and use 3D convnet architectures. For instance,~\citet{skalic2019shape} and~\citet{wang2022relation} train models to generate ligand SMILES from voxelized pocket structures.~\citet{ragoza2022generating} uses conditional variational autoencoder to generate 3D ligands conditioned on voxelized pocket structures.~\citet{long2022zero} encodes molecular shapes with voxels and propose a model that generates molecular graphs conditioned on the voxelized shapes.
Our method also utilizes voxelized molecules, but differs from previous work in terms of how we voxelize molecules, the network architecture and the generative model we use.

More recent methods, however, use point-cloud representations and SE(3)-equivariant graph neural networks.~\citet{luo2022autoregressive,liu2022generating,peng2022pocket2mol} propose conditional autoregressive methods that add atoms (and their corresponding bonds) one at a time, while~\citet{rozenberg2023structure} propose a method based on normalizing flows.
Other authors propose methods that generate molecules by iteratively adding fragments conditioned either on ligand~\citep{adams2022equivariant} or pocket~\citep{zhang2023molecule,powers2023geometric} shapes.

Current state of the art models for SBDD are based on point-cloud diffusion models. DiffSBDD~\citep{schneuing2022structure} and TargetDiff~\citep{guan20233tdiff} adapt EDMs to the pocket-conditional generation setting. DecompDiff~\citep{guan23ddiff} extends TargetDiff with three modifications: they decompose ligand into fragment priors, apply diffusion process on bonds (in addition to atom types and coordinates), and add guidance during generation process. 

This work shows that we don't necessarily need point clouds or SE(3)-equivariant networks to achieve competitive results on structure-based drug design tasks. As we demonstrate in this paper, voxel-based representations---when coupled with a powerful generative model and an expressive network---can achieve state-of-the-art results.

\subsection{Conditional image generation}\label{sec:cond_img_gen}
Conditional image generation has been extremely successful and widely applied to many different computer vision tasks: text-to-image generation~\citep{rombach2022high}, super-resolution~\citep{saharia2022image}, optical flow and depth estimation~\citep{saxena2023surprising}, colorization~\citep{saharia2022palette}, in-painting~\citep{lugmayr2022repaint} and semantic segmentation~\citep{amit2021segdiff}. All these methods share a commonality that inspired our approach: 
they adapt score-based generative models to the conditional setting by modifying the very expressive U-Net architecture~\citep{ronneberger2015u} to incorporate the conditioning signal.

\section{Method}\label{sec:method}
\subsection{Voxelized molecules}
We represent atoms as spherical densities in 3D space decaying exponentially with the square distance to their atomic center (see~\eqref{eq:atomic_density} in the appendix). Voxelized molecules are created by discretizing the space around atoms into voxel grids, where the value at each voxel represents atomic occupancy. 
Voxels take values between 0 (far away from all atoms) and 1 (at the center of atoms).
Ligands and pockets are each represented as cubic grids with side length $L\in\mathbb{N}$. 
Each ligand grid and its corresponding pocket grid are centered around the center of mass of the ligand.
We assume ligands have $\cligand$ atom types and pockets have $\cpocket$. Each atom type (element) is represented by a different grid channel (similar to R,G,B channels of images).

We consider a dataset with $n$ voxelized ligand-pocket binding pairs $\{(\ligand, \pocket)_i\}_{i=1}^n$, where $\ligand\in\mathbb{R}^{\dligand}$ is the voxelized ligand and $\pocket \in\mathbb{R}^{\dpocket}$ the voxelized pocket with dimensions $\dligand=\cligand L^3$ and $\dpocket=\cpocket L^3$, respectively. 
Following previous work~\citep{pinheiro2023voxmol}, we consider a fixed grid resolution of .25\AA~and a grid length $L=64$. Therefore, the grids in the dataset occupy a volume of $16^3$ cubic \AA ngstr\"oms. See~\autoref{sec:appendix_vox_mols} (appendix) for further details on the voxelization procedure. 

\begin{figure*}[t]
\begin{center}
\centerline{\includegraphics[width=\linewidth]{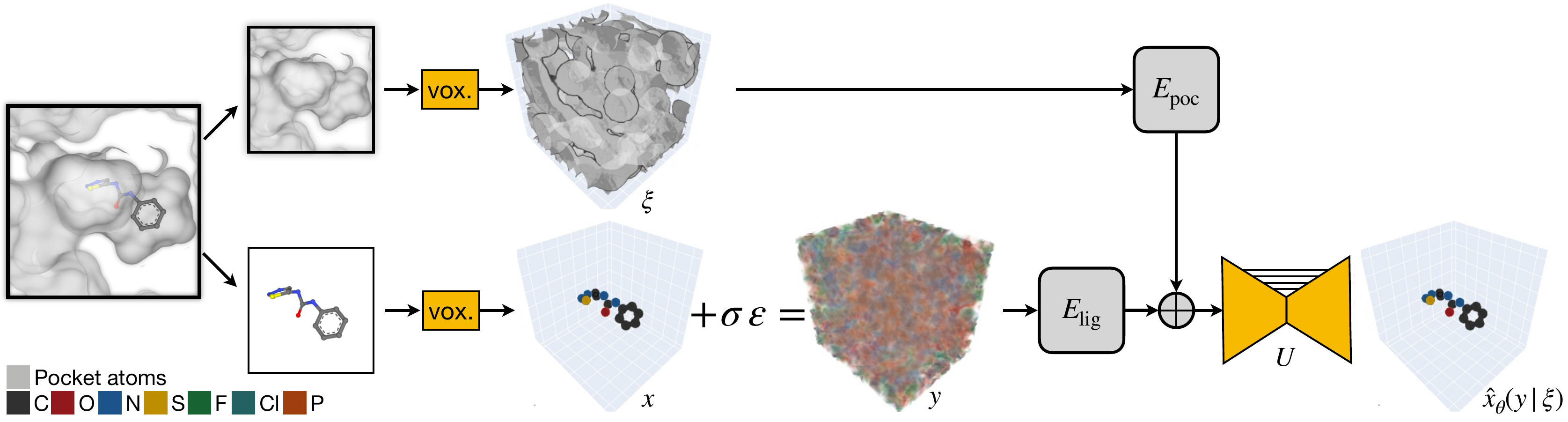}}
\vskip -0.1in
\caption{Conditional denoiser architecture. Given a ligand-pocket complex sample, we discretize each molecule resulting in the voxelized ligand $\ligand$ and pocket $\pocket$. The ligand is corrupted by Gaussian noise with noise level $\sigma$. The corrupted ligand and pocket are encoded into a common embedding space (with the same spatial dimensions as the inputs) with encoders $E_{\rm lig}$ and $E_{\rm poc}$, respectively. The two representations are added together and forwarded through a 3D U-Net $U$ to recover the clean version of the ligand. To facilitate visualization, we threshold the grid values, $\hat{x}=\mathbbm{1}_{\ge.1}(\hat{x})$. }
\label{fig:architecture}
\end{center}
\vskip -0.2in
\end{figure*}

\subsection{Conditional neural empirical Bayes} 
Let $\Ligand$ and $\Pocket$ be the random variables associated with voxelized ligands and pockets, taking values in $\mathbb{R}^{\dligand}$ and $\mathbb{R}^{\dpocket}$, respectively, with the joint density $p(\ligand,\pocket)$. Given $\Pocket=\pocket$, we would like to learn the smoothed conditional score function $\nabla_{\noisyligand} \log p(\noisyligand|\pocket)$, where the smoothed conditional density $p(\noisyligand|\pocket)$ is obtained by convolving $p(\ligand|\pocket)$ with an isotropic Gaussian kernel:
\begin{equation} \label{eq:noise-factorization} p(\noisyligand|\pocket) = \int p(\noisyligand|\ligand)p(\ligand|\pocket) d\ligand,
\end{equation}
where 
\begin{equation} \label{eq:gauss}
	p(\noisyligand|\ligand) = \frac{1}{(2\pi \sigma^2)^{\dligand/2}} \exp\Bigl(-\frac{\Vert \noisyligand- \ligand \Vert^2}{2 \sigma^2} \Bigr),
\end{equation}
equivalently $\noisyLigand|\ligand \sim \normal(\ligand,\sigma^2 I_{\dligand})$. The factorization \eqref{eq:noise-factorization} is mathematically equivalent to setting up the noise process such that conditioned on the ligand $\ligand$, the noisy ligand $\noisyligand$ is independent of the pocket $\pocket$:   $p(\noisyligand|\ligand,\pocket) = p(\noisyligand|\ligand)$. Below we see that due to this conditional independence the empirical Bayes results used in~\citep{saremi2019neural} readily generalize to our conditional setting. 

In particular, we would like to relate $\hatligand(\noisyligand|\pocket)$, the conditional least-squares Bayes estimator of $\Ligand$ given $(\noisyligand,\pocket)$, to the smoothed conditional score function $g(\cdot|\pocket)$, defined by \begin{equation} \label{eq:conditional-score}
    g(\noisyligand|\pocket) = \nabla_{\noisyligand} \log p(\noisyligand|\pocket).
\end{equation}
Due to our noise process given by \eqref{eq:noise-factorization}, the conditional Bayes estimator $\hatligand(\noisyligand|\pocket)=\mathbb{E}[\Ligand|\noisyligand,\pocket]$ is given by
\begin{equation} \label{eq:conditonal-bayes-estimator}
    \hatligand(\noisyligand|\pocket) = \frac{\int \ligand p(\noisyligand|\ligand)  p(\ligand|\pocket)  d\ligand}{\int p(\noisyligand|\ligand) p(\ligand|\pocket) d\ligand}.
\end{equation}
This leads to the following:
\begin{proposition} \label{prop:conditional-tweedie} Given the noise process \eqref{eq:noise-factorization} and \eqref{eq:gauss}, the conditional Bayes estimator \eqref{eq:conditonal-bayes-estimator} can be written in closed form in terms of the conditional score function \eqref{eq:conditional-score}: 
	\begin{equation} \label{eq:conditional-tweedie}
  \hatligand(\noisyligand|\pocket) =  \noisyligand + \sigma^2 g(\noisyligand|\pocket).
\end{equation}
\end{proposition}

\begin{proof}
The derivation is simple and similar to the known classical results~\citep{robbins1956empirical}, but we provide it here for completeness. We start with the following property of the Gaussian kernel \eqref{eq:gauss}:
$$\sigma^2 \nabla_{\noisyligand} p(\noisyligand|\ligand) = (\ligand-\noisyligand) p(\noisyligand|\ligand).$$
We then multiply both sides of the above equation by $ p(\ligand|\pocket)$ and integrate over $\ligand$. Since $\nabla_{\noisyligand}$ and the integration over $\ligand$ commute, using \eqref{eq:noise-factorization}, we arrive at
$$
 \sigma^2 \nabla_{\noisyligand} p(\noisyligand|\pocket) = \int \ligand  p(\noisyligand|\ligand)  p(\ligand|\pocket)  d\ligand - \noisyligand  p(\noisyligand|\pocket).
$$
Eq. \eqref{eq:conditional-tweedie} follows through by dividing both sides of the above equation by $p(\noisyligand|\pocket)$ and using \eqref{eq:noise-factorization}, \eqref{eq:conditional-score}, and \eqref{eq:conditonal-bayes-estimator}.\footnote{The relation \eqref{eq:conditional-tweedie} between the conditional estimator and the conditional score function is the conditional form of what is referred to in the literature as Tweedie's formula, first derived for Gaussian kernels by \citet{miyasawa1961empirical}.} 
\end{proof}

Due to this extension, summarized by \autoref{prop:conditional-tweedie}, the (unconditional) neural empirical Bayes framework~\citep{saremi2019neural} can be readily adopted for both learning the  conditional score function $\nabla_y \log p(\noisyligand|\pocket)$ and drawing approximate samples from $p(\ligand|\pocket)$. By learning the conditional score function with a least-squares denoising objective we are able to sample from $p(\noisyligand|\pocket)$ using Langevin MCMC (\emph{walk}), and also able to estimate clean ligands conditioned on the pocket using the Bayes estimator (\emph{jump}), thus drawing approximate samples from $p(\ligand|\pocket)$. We refer to this sampling scheme  as \emph{conditional walk-jump sampling} (cWJS). The details on both learning and sampling are given next.


\subsection{Conditional denoiser}
From the perspective of learning $p(\noisyligand|\pocket)$, the key property of \eqref{eq:conditional-tweedie} is the appearance of the \emph{score function}~\citep{hyvarinen2005estimation} in the right hand side. This means that in setting up the learning objective we do not have to worry about the intractable partition function~\citep{saremi2019neural}. In this paper  we parametrize the conditional denoiser, \ie, the left hand side of \eqref{eq:conditional-tweedie}, from which the conditional score function is derived.

There are many  ways to construct a conditional voxel denoiser. In this work we focus on a network design that is simple, scalable, and shows good empirical results.
We follow the conditional image generation literature (see \autoref{sec:cond_img_gen}) and propose the following architecture: (i) encode the noisy ligand and the pocket with separate encoders, (ii) merge their representations and (iii) pass through an encoder-decoder architecture to predict clean sample. More precisely:
\begin{equation}
\hatligand_\theta(\noisyligand|\pocket) = 
    U_{\theta_3}(E_{\rm lig}(\noisyligand;\theta_1) + E_{\rm poc}(\pocket;\theta_2)),
    \label{eq:arquitecture}
\end{equation}
where $\hat{x}_\theta:\mathbb{R}^{\dligand}\times\mathbb{R}^{\dpocket}\rightarrow\mathbb{R}^{\dligand}$ is parameterized by $\theta=(\theta_1,\theta_2,\theta_3)$, \encoder and \decoder encode the noisy ligand $\noisyligand$ and pocket $\pocket$ to a common representation $\mathbb{R}^{d_e}$, $d_e=c_eL^3$, with the same spatial dimensions as input, and $U$ is a 3D U-Net. \autoref{fig:architecture} shows an overview of the model architecture.

\begin{figure*}[!th]
\begin{center}
\centerline{\includegraphics[width=\linewidth]{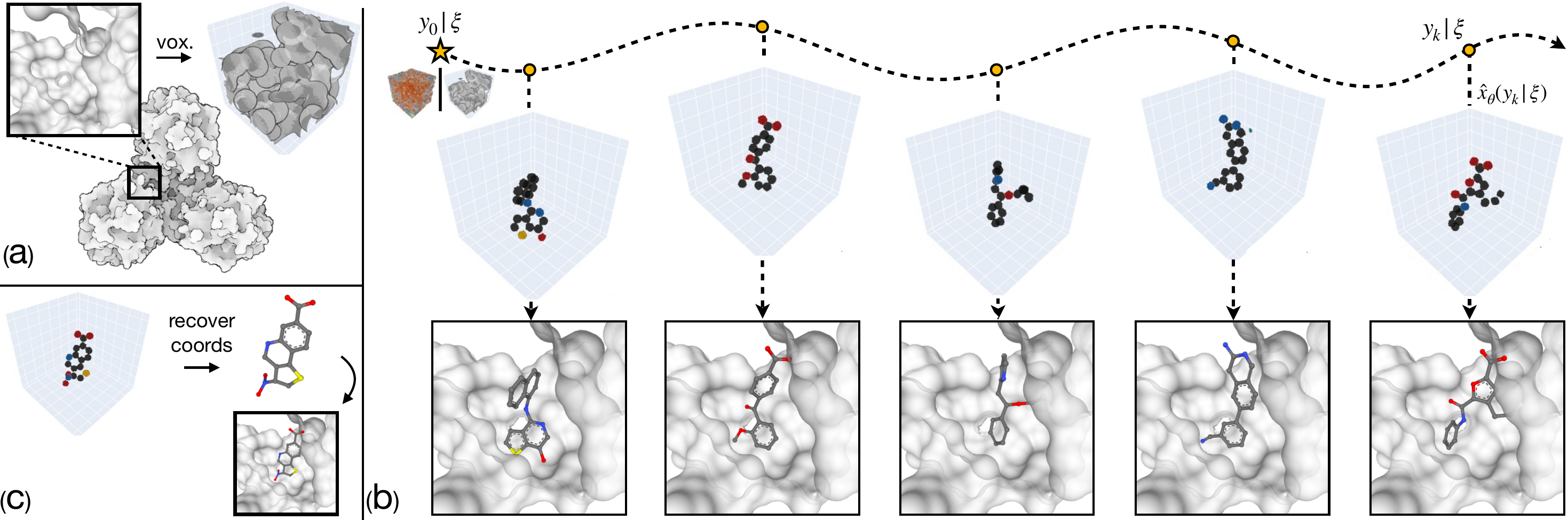}}
\vskip -0.1in\caption{Illustration of pocket-conditional walk-jump sampling chain. (a) First, we voxelize a given protein binding pocket. (b) Then, we sample noisy voxelized ligands (given the pocket) with Langevin MCMC and estimate clean samples with the estimator. (c) Finally, we recover the atomic coordinates from voxel grids. In this figure, jumps are done at every $\Delta k=100$ walk steps.}
\label{fig:wjs}
\end{center}
\vskip -0.2in
\end{figure*}

Training samples are generated by first sampling a ligand-pocket complex from the empirical distribution and voxelize each of them, resulting in the pair $(\ligand,\pocket)$. We add Gaussian noise drawn from $\normal(0,\sigma^2 I_{\dligand})$ to the ligand: $\noisyligand = \ligand + \sigma\, \varepsilon,\ \varepsilon\sim\normal(0,I_{d_x})$.
The denoising network \eqref{eq:arquitecture} takes the pair $(\noisyligand, \pocket)$ as input and should output $\ligand$, the clean ligand. 
The model is trained by minimizing the mean square error over all voxels in all voxelized ligands, \ie, by minimizing the following loss:
\begin{equation}
    \mathcal{L}(\theta)= \mathbb{E}_{
       \substack{(\ligand,\pocket)\sim p(\ligand,\pocket), \\ \varepsilon \sim \normal(0,I_{\dligand})}
    }
    \big\Vert \hat{x}_\theta(\ligand + \sigma \varepsilon | \pocket) - \ligand \big\Vert^2.  
    \label{eq:loss}
\end{equation}
Finally, following learning, we can approximate the score function of the pocket-conditional distribution using~\eqref{eq:conditional-tweedie}:
\begin{equation}
    g_\theta(\noisyligand| \pocket)=\frac{1}{\sigma^2} (\hatligand_\theta(\noisyligand|\pocket) - \noisyligand).    \label{eq:param_score_funcion}
\end{equation}

\subsection{Conditional walk-jump sampling}
\begin{algorithm}[t]
\caption{cWJS, using the BAOAB scheme~\citep[Algorithm 1]{sachs2017langevin} adapted for the conditional ULD \eqref{eq:lmcmc}. 
}
  \begin{algorithmic}[1]
    \STATE \textbf{Input} Learned conditional estimator $\hat{x}_\theta$, score function $g_\theta$ (from~\eqref{eq:param_score_funcion}), noise level $\sigma$, protein pocket $\pocket$
    \STATE \textbf{Input} step size $\delta$, friction $\gamma$, steps taken $K$ 
    \STATE \textbf{Output} $\widehat{x}$
    \STATE $y_0\sim \normal(0,\sigma^2I_{\dligand})+\uniform_{\dligand}(0,1)$ 
    \STATE $v_0 \leftarrow 0$    
    
    \FOR{$k=0,\dots, K-1$} 
      \STATE $y_{k + 1}  \leftarrow y_k + \frac{\delta}{2} v_{k}$
      \STATE $g   \leftarrow  g_{\theta}(y_{k+1}|\pocket)  $
      \STATE $v_{k +1} \leftarrow v_k  + \frac{\delta}{2} g$
      \STATE $\varepsilon\sim \normal(0,I_{\dligand}) $ 
      \STATE $v_{k+1} \leftarrow {\rm e}^{-\gamma\delta} v_{k+1} + \frac{\delta}{2} g + \sqrt{\left(1-{\rm e}^{-2 \gamma\delta }\right)} \varepsilon$
      \STATE $y_{k+1}\leftarrow y_{k+1} + \frac{\delta}{2} v_{k+1}$
    \ENDFOR
    \STATE $\hat{x}  \leftarrow \hat{x}_\theta(y_K|\pocket)$
  \end{algorithmic}
  \label{alg:wjs}
\end{algorithm}
We leverage the estimator~\eqref{eq:arquitecture}, and the learned score function~\eqref{eq:param_score_funcion} to sample voxelized ligands conditioned on (voxelized) protein pockets. 
This is done by  adapting the walk-jump  scheme~\citep{saremi2019neural, pinheiro2023voxmol, frey2023learning} to the conditional setting, resulting in the following two-step approach, which we call conditional walk-jump sampling (cWJS) (see Algorithm~\ref{alg:wjs}):
\begin{enumerate}[(i)]
\item \emph{(walk step)}
Sample noisy ligands from $p(\noisyligand|\pocket)$ with Langevin MCMC. We consider the conditional form of the underdamped Langevin diffusion (ULD):  
\begin{equation}
    \begin{split} 
    dv_t &= -\gamma v_t dt + g_\theta(\noisyligand_t|\pocket)dt + \sqrt{2\gamma}\,dB_t,\\
    d\noisyligand_t &= v_t dt,
    \end{split}
    \label{eq:lmcmc}
\end{equation}
where $(dB_t)_{t\geq 0}$ is the standard Brownian motion in $\mathbb{R}^{\dligand}$ and $\gamma$ is the friction. We use the discretization algorithm proposed by~\citet{sachs2017langevin} to generate samples $y_k$ (the inner loop of Algorithm~\ref{alg:wjs}).
\item\emph{(jump step)} 
At an arbitrary time step $K$, typically after a burn-in time, we estimate clean molecules with the conditional estimator, \ie, $\hatligand_K=\hatligand_\theta(y_K|\pocket)$. Since jumps do not interact with the walks, this step can be moved to inner loop at the cost of memory (storing clean samples). More explanations are given below.
\end{enumerate}

\renewcommand{\arraystretch}{1.2} 
\renewcommand\cellset{\renewcommand\arraystretch{0.3}%
\setlength\extrarowheight{10pt}}

\begin{table*}[!th]
  \caption{Results on CrossDocked2020 test set. Arrows $\uparrow$/$\downarrow$ denote that higher/lower numbers are better, respectively. For the last column, numbers close to Reference are better. Baseline results are taken from~\citep{guan23ddiff}. For each metric, we {\bf bold} and \underline{underline} the best and second best methods, respectively. Our results are shown with mean/standard deviation across 3 runs. $^{\dagger}$This method uses different training assumptions and relies on (external) rule-based algorithms to generate samples.}\label{tab:xdocked_res}
  \resizebox{1.\textwidth}{!}{%
  \begin{tabular}{l | cc | cc | cc | cc | cc | cc | cc | cc}
    & \multicolumn{2}{c|}{\small VinaScore ($\downarrow$)} & \multicolumn{2}{c|}{\small VinaMin ($\downarrow$)} & \multicolumn{2}{c|}{\small VinaDock ($\downarrow$)} & \multicolumn{2}{c|}{\small High aff. (\%, $\uparrow$)} & \multicolumn{2}{c|}{\small QED ($\uparrow$)} & \multicolumn{2}{c|}{\small SA ($\uparrow$)} & \multicolumn{2}{c|}{\small Diversity ($\uparrow$)} & \multicolumn{2}{c}{\small \# atoms / mol.}\\
    
    & {\small Avg.} & {\small Med.} & {\small Avg.} & {\small Med.} &  {\small Avg.} & {\small Med.} & {\small Avg.} & {\small Med.} &{\small Avg.} & {\small Med.} & {\small Avg.} & {\small Med.} & {\small Avg.} & {\small Med.} & {\small Avg.} & {\small Med.} \\
    \shline
    \emph{Reference} & -6.36 & -6.46 & -6.71 & -6.49 & -7.45 & -7.26 & - & - & .48 & .47 & .73 & .74 & - & - & 22.8 & 21.5 \\
    \hline
    AR & -5.75 & -5.64 & -6.18 & -5.88 & -6.75 & -6.62 & 37.9 & 31.0 & .51 & .50 & .63 & .63 & .70 & .70 & 17.6 & 16.0 \\
    
    Pocket2mol & -5.14 & -4.70 & -6.42 & -5.82 & -7.15 & -6.79 & 48.4 & 51.0 & \underline{.56} & \underline{.57} & {\bf .74} & {\bf .75} & .69 & .71 & 17.7 & 15.0   \\
    DiffSBDD & -1.94 & -4.24 & -5.85 & -5.94 & -7.00 &-6.90  & 46.3 & 47.2 & .48 & .48 & .58 & .57 & \underline{.73} & .72& 24.0 & \underline{23.0}\\
    
    TargetDiff & -5.47 & -6.30 & -6.64 & -6.83 & -7.80 & -7.91 & 58.1 & 59.1 & .48 & .48 & .58 & .58 & .72 & .71 & 24.2 & 24.0 \\
    
    DecompDiff$^{\dagger}$ & -5.67 & -6.04 & -7.04 & -7.09 & {\bf -8.39} & {\bf-8.43} & \underline{64.4} & \underline{71.0} & .45 & .43 & .61 & .60 & .68 & .68 & 20.9 & {\bf 21.0} \\

    \hline

    \modelname $_{\sigma=0.9}$& \makecell{{\bf -6.94}\\\tiny{($\pm$.00)}} & \makecell{{\bf -7.11}\\\tiny{($\pm$.04)}} & \makecell{{\bf -7.54}\\\tiny{($\pm$.01)}} & \makecell{{\bf -7.55}\\\tiny{($\pm$.04)}} & \makecell{ \underline{-8.30}\\\tiny{($\pm$.02)}} & \makecell{\underline{-8.41}\\\tiny{($\pm$0.1)}} & \makecell{{\bf 71.3}\\\tiny{($\pm$2.4)}} & \makecell{{\bf 71.5}\\\tiny{($\pm$.00)}} & \makecell{{\bf .57}\\\tiny{($\pm$.00)}} & \makecell{{\bf .58}\\\tiny{($\pm$.00)}} & \makecell{\underline{.70}\\\tiny{($\pm$.00)}} & \makecell{\underline{.69}\\\tiny{($\pm$.00)}} & \makecell{\underline{.73}\\\tiny{($\pm$.00)}} & \makecell{\underline{.74}\\\tiny{($\pm$.00)}} & \makecell{{\bf 23.4}\\\tiny{($\pm$.1)}} & \makecell{24.0\\\tiny{($\pm$.0)}}   \\  

    \modelname $_{\sigma=1.0}$& \makecell{\underline{-6.63}\\\tiny{($\pm$.03)}} & \makecell{\underline{-6.70}\\\tiny{($\pm$.02)}} & \makecell{\underline{-7.12}\\\tiny{($\pm$.02)}} & \makecell{\underline{-7.18}\\\tiny{($\pm$.03)}} & \makecell{ -7.82\\\tiny{($\pm$.01)}} & \makecell{-7.89\\\tiny{($\pm$0.01)}} & \makecell{{58.3}\\\tiny{($\pm$0.2)}} & \makecell{{61.5}\\\tiny{($\pm$.00)}} & \makecell{{ .55}\\\tiny{($\pm$.03)}} & \makecell{\underline{.57}\\\tiny{($\pm$.00)}} & \makecell{.69\\\tiny{($\pm$.00)}} & \makecell{\underline{.69}\\\tiny{($\pm$.00)}} & \makecell{{\bf .75}\\\tiny{($\pm$.00)}} & \makecell{{\bf .76}\\\tiny{($\pm$.00)}} & \makecell{\underline{21.7}\\\tiny{($\pm$.1)}} & \makecell{{\bf 22.0}\\\tiny{($\pm$.0)}}   \\  
    
  \end{tabular}
  }
\end{table*}
As we alluded to above, a key property of cWJS (which is borrowed from WJS) is that the jumps are decoupled from the walks~\citep{saremi2019neural}. Therefore, we can estimate samples at any arbitrary step of the MCMC chain. This is in contrast to diffusion models, where every sample requires a full MCMC chain to go from noise to sample. \autoref{fig:wjs} illustrates the generation process starting from noise, estimating clean molecules (given the pocket) at each 100 steps.

Algorithm~\ref{alg:wjs} describes a simple implementation of the (``\emph{de novo}'') conditional walk-jump sampling method: initialize a chain (lines 4-5), walk $K$ steps (lines 6-13) and jump (line 14) to estimate sample $\hat{x}$, given pocket $\pocket$. We use the discretization scheme by~\citep[Algorithm 1]{sachs2017langevin}, which requires the step size $\delta$. Due to the decoupling of the jumps, line 14 in the algorithm can also be evaluate at the end of the inner loop $\hat{x}_{k+1} \leftarrow \hat{x}_\theta(\noisyligand_{k+1}|\pocket)$ (or preferably not at every step, but at some desired frequency $\Delta k$), \eg, in \autoref{fig:wjs} we set $\Delta k = 100$.

Finally, we run the same peak detection post-processing of VoxMol~\citep{pinheiro2023voxmol} to recover the atomic coordinates from the generated voxelized ligand. See \autoref{sec:appendix_sampling} (appendix) for more implementation details on how we perform sampling.

\section{Experiments}\label{sec:experiments}

\subsection{Experimental setup}
\paragraph{Dataset.}
We benchmark our model on CrossDocked2020~\citep{francoeur2020three}, a popular dataset for pocket-conditional molecule generation. 
We follow previous work~\citep{schneuing2022structure,guan20233tdiff,guan23ddiff} and use the pre-processing and splitting proposed by ~\citep{luo20213d}: the cross-docked dataset is reduced from 22.5M to 100,100 pairs of ``high-quality'' ligand-pocket pairs. The pockets are clustered (sequence identity $<30\%$ using MMseqs2~\citep{steinegger2017mmseqs2}) and 100,000 and 100 ligand-pocket pairs are partitioned into the training and test sets, respectively. We take 100 samples from the training set as our hold-out validation set. 

Ligands are represented with seven chemical elements (C, O, N, S, F, Cl and P) and protein pockets with four (C, O, N, S). Both molecules are modeled with implicit hydrogens. In the case of protein pockets, we consider all heavy atoms of every amino acid. Each ligand and its associated pocket are centered around the center of mass of the ligand and discretized on a cubic grid with length 64, resulting in tensors of dimensions $\mathbb{R}^{7\times64\times64\times64}$ and $\mathbb{R}^{4\times64\times64\times64}$ for ligand and pocket, respectively. During training, we apply random translation (uniform value between $[-1,1]$ \AA~on 3D coordinates) and rotation (uniform value between $[0,2\pi)$ on three Euler angles) to each ligand-protein training sample. These augmentations are applied to both molecules so that their relative pose remains unchanged. We show the importance of data augmentation during training on the ablation study presented on~\autoref{sec:ablation_data_aug} (appendix).

\begin{figure*}[!th]
    \begin{center}
    \centerline{\includegraphics[width=\linewidth]{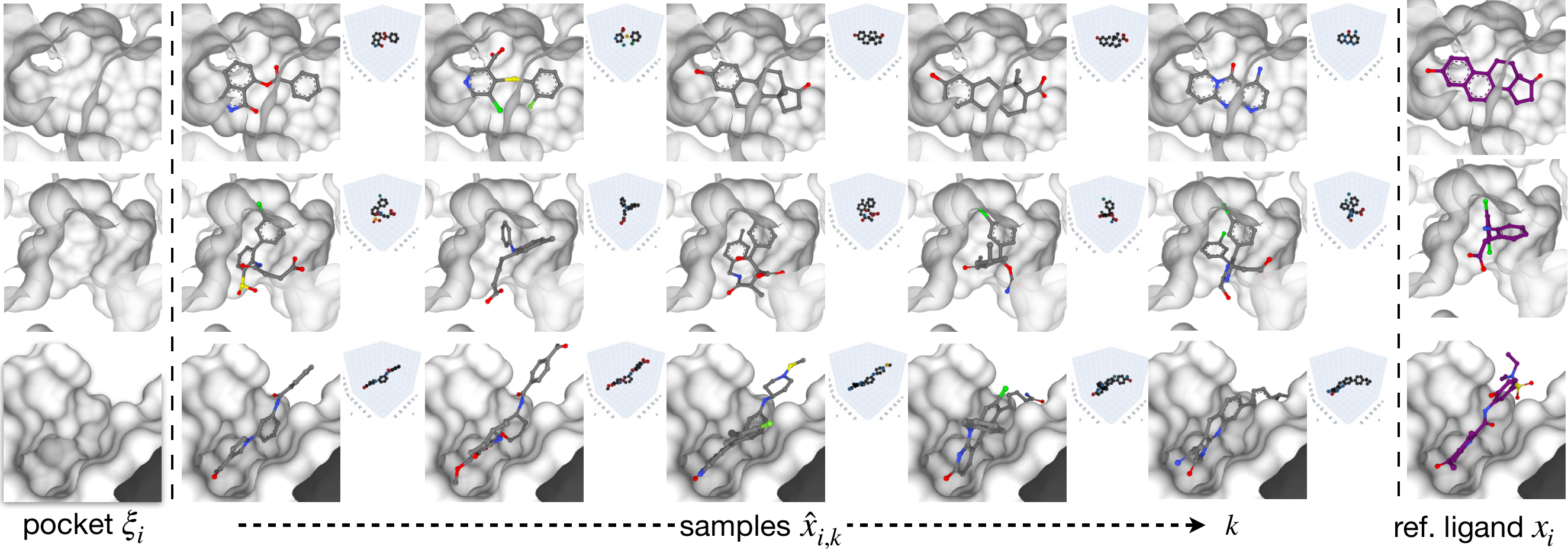}}
    \vskip -0.1in
    \caption{Example of generated ligands $\hat{x}_{i,k}$ given pocket $\pocket_i$. Each row represents a single chain of samples for a given protein pocket (\texttt{1E3R}, \texttt{2I2Z}, \texttt{5CRZ} from top to bottom). For each generated sample, we show the ligand-pocket complex and the generated voxelized molecule. The samples from each row are generated from the same MCMC chain. The provided ground-truth ligands are shown on the last column.}
    \label{fig:qualitative_res}
    \end{center}
    \vskip -0.2in
\end{figure*}

\vspace{-.1in}\paragraph{Architecture.}
Encoders \encoder and \decoder map the voxelized molecules to a common embedding space with latent dimension $c_e=16$ and same spatial dimensions, \ie, in $\mathbb{R}^{d_e},d_e=c_e\times64^3$. The encoders contain two residual blocks, each having two padded $3\times3\times3$ convolutional layers with 16 units followed by SiLU non-linearities~\citep{elfwing2018sigmoid}. The two embeddings are added together and the merged representation goes through a 3D U-Net architecture~\citep{ronneberger2015u} which predict the clean ligand, of dimension $\dligand=7\times64^3$. We use the same 3D U-Net architecture as in~\citep{pinheiro2023voxmol}, only modifying the dimensionality of the input and output.
We train two versions of the model, one with noise level $\sigma=0.9$ and another with $\sigma=1.0$.

The models are trained with batch size of 64, learning rate $10^{-5}$ and weight decay $10^{-2}$. The weights are updated with AdamW optimizer~\citep{loshchilov2019decoupled} ($\beta_1=.9$, $\beta_2=.999$) and exponential moving average with decay $.999$. Both models are trained for 340 epochs on four NVIDIA A100 GPUs (a total of ten days). 
For benchmarking purposes, we generate samples by repeatedly applying Algorithm~\ref{alg:wjs}, with $K=500$ steps. We use $\delta=\sigma/2$ and fix $\gamma=1/\delta$, \ie, the ``effective friction''~\citep{saremi2023universal} is set to $1$.

\vspace{-.1in}\paragraph{Baselines.}
We compare our method with various recent benchmarks for pocket-conditional ligand generation: two autoregressive models on point clouds (\emph{AR}~\citep{luo20213d} and \emph{Pocket2Mol}~\cite{peng2022pocket2mol}) and three diffusion models applied on point clouds (\emph{DiffSBDD}~\citep{schneuing2022structure}, \emph{TargetDiff}~\citep{guan20233tdiff} and \emph{DecompDiff}~\citep{guan23ddiff}).
Unlike other baselines, DecompDiff (the current state of the art) relies on a rule-based algorithm (AlphaSpace2, proposed by~\citet{rooklin2015alphaspace}) to compute subpockets and decomposable priors during sampling. 
All methods with the exception of DecompDiff---including \modelname---uses OpenBabel~\citep{o2011obabel} to assign bonds from generated atom coordinates.

\vspace{-.1in}\paragraph{Metrics.}
We evaluate performance using similar metrics as previous work\footnote{We use the implementation of \url{https://github.com/guanjq/targetdiff}.}. For each method, we sample 100 ligands for each of the 100 protein pockets. We measure affinity with three metrics using AutoDock Vina~\citep{eberhardt2021autodock}: \emph{VinaScore} computes the binding affinity (docking score) of the molecule on its original generated pose, \emph{VinaMin} performs a local energy minimization on the ligand followed by dock scoring, \emph{VinaDock} performs full re-docking (search and scoring) between ligand and target. \emph{High affinity} computes the percentage of generated molecules that has lower VinaDock score than the reference ligand.
We measure drug-likeness, \emph{QED}~\citep{bickerton2012quantifying}, and synthesizability, \emph{SA}~\citep{ertl2009estimation}, scores of generated molecules with RDKit~\citep{landrum2016RDKit}. \emph{Diversity} is computed by averaging the Tanimoto distance (in RDKit fingerprints) of every pair of generated ligands per pocket~\citep{bajusz2015tanimoto}. The \emph{\# atoms / molecules} is the average number of (heavy) atoms per molecules.

We also compute metrics proposed in PoseCheck~\citep{harris2023benchmarking}\footnote{Code available at \url{https://github.com/cch1999/posecheck}.}. \emph{Steric clash} computes the number of clashes between generated ligands and their associated pockets. \emph{Strain energy} measures the difference between the internal energy of generated molecule's pose (without pocket) and a relaxed pose (computed using Universal Force Field~\citep{rappe1992uff} within RDKit). 
\emph{Interaction fingerprints} describe intramolecular interactions between the generated ligands and protein pockets. See~\citep{harris2023benchmarking} for more details about these metrics.

\subsection{Results}
\paragraph{Binding affinity.}
\autoref{tab:xdocked_res} compares models on CrossDocked2020 \emph{test set} in terms of binding affinity and molecular properties. For each metric, we compute the mean and the median over all generated molecules. The row \emph{Reference} shows results of ground-truth ligands provided on the test set.

\begin{figure}[!t]
    \begin{center}
    \centerline{\includegraphics[width=1.\linewidth]{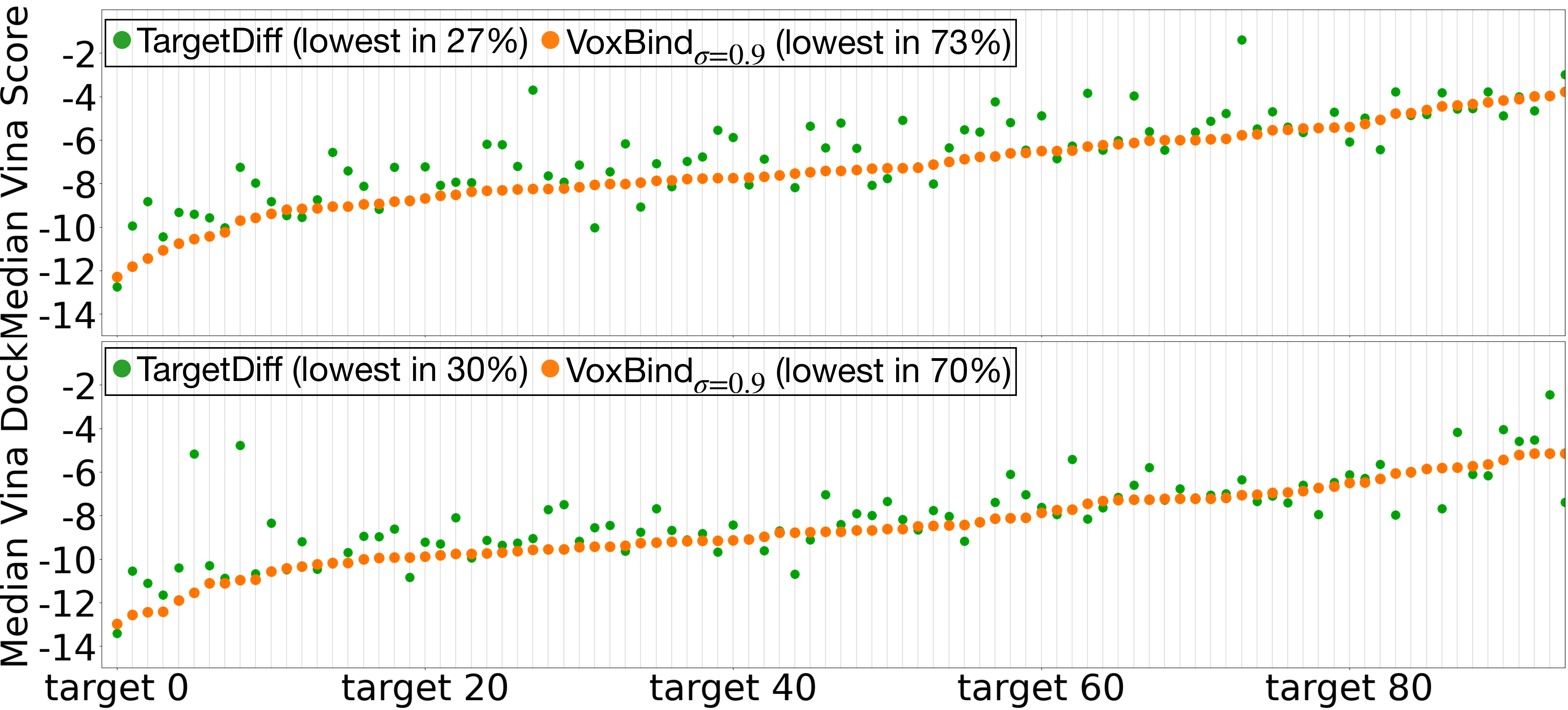}}
    \vskip -0.1in
    \caption{Median VinaScore and VinaDock (score on generated and redocked poses, respectively) of all generated molecules for each target on the test set (lower is better). Pockets are sorted by \modelname$_{\sigma=0.9}$'s score.}
    \label{fig:median_vina_dock}
    \end{center}
    \vskip -0.2in
\end{figure}

Our models achieve better results than TargetDiff (the best model trained under same assumptions) and DecompDiff (which relies on different training assumptions) in most metrics evaluated. The molecules generated by \modelname$_{\sigma=0.9}$ and \modelname$_{\sigma=1.0}$ have \emph{better binding affinity}. In particular, our models generate molecules with \emph{better pose} than other methods: their docking score with generated pose (VinaScore) is better/similar to that of methods after energy minimization (VinaMin).

\autoref{fig:median_vina_dock} further compares our models with TargetDiff. It shows the median VinaScore and VinaDock (generated and redocked poses, respectively) of all generated molecules for each target on the test set. The generated molecules from \modelname$_{\sigma=0.9}$ have better (computational) affinity on 73\% of tested protein pockets.

\vspace{-.1in}\paragraph{Molecular properties.}
\autoref{tab:xdocked_res} also shows results on molecular properties. \modelname~achieves \emph{better QED and SA} than diffusion-based models. The molecules generated by our models are particularly \emph{more diverse} than baselines.

We note that our models \emph{capture better the distribution of number of atoms per molecules} than the baselines. The autoregressive models generate molecules considerably smaller than the reference, while the diffusion model baselines tend to sample larger molecules. Our models achieve similar numbers as DecompDiff without the extra sampling assumptions. Unlike diffusion-based approaches, \emph{our model captures this distribution implicitly and does not require sampling the number of atoms beforehand}.

We also compare models in terms of ring statistics on \autoref{fig:qualitative_res} (appendix): we show the histogram of number of rings per molecule (top), the histograms of ring sizes (\ie, number of atoms per ring) (mid) and the histograms of fraction of aromatic atoms per molecule (bottom) for all generated molecules from different methods. In all cases, our models capture reference ring properties more favourably than baselines.

\begin{figure}[!t]
    \begin{center}
    \centerline{\includegraphics[width=\linewidth]{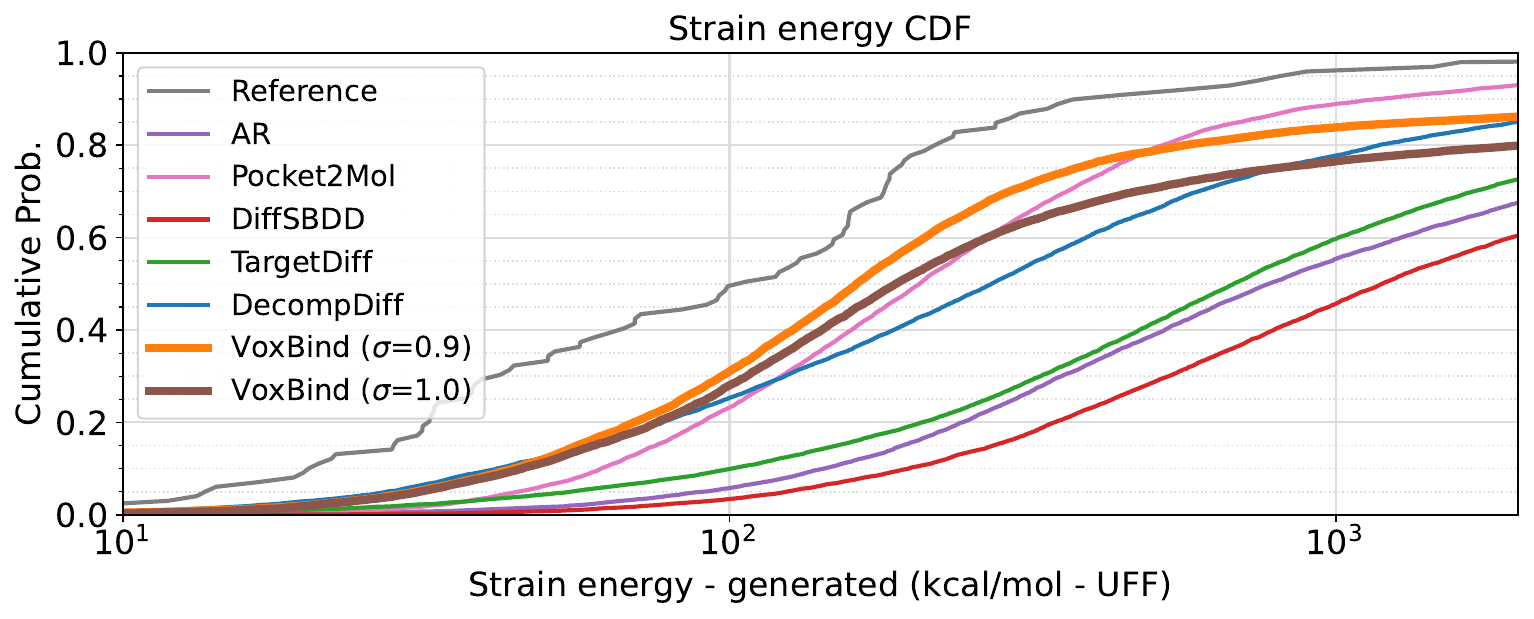}}
    \vskip -0.1in
    \caption{The cumulative distribution function of strain energy of molecules on their generated pose.}
    \label{fig:strain_energies}
    \end{center}
    \vskip -0.2in
\end{figure}

\vspace{-.1in}\paragraph{Molecular structures.}
First, we analyse the strain energies of molecules (as defined by PoseCheck~\citep{harris2023benchmarking}). \autoref{fig:strain_energies} shows the cumulative distribution function (CDF) of strain energies (on the raw generate poses) for molecules from different models. The reference set has median strain energy of $102.5$ kcal/mol, while the three best performing models, Pocket2Mol, \modelname$_{\sigma=0.9}$  and \modelname$_{\sigma=1.0}$ have $205.8$, $161.9$ and $188.3$ kcal/mol, respectively.
\autoref{fig:energy_min_and_rigid_rmsd}(a) (appendix) shows the CDF of strain energies after local minimization is performed on each molecule. The strain energies are largely reduced, but conclusions remain unchanged: \modelname's molecules has \emph{lower strain energies} than those generated by diffusion models.
Additionally, \autoref{fig:energy_by_rotatable_generated} and \autoref{fig:energy_by_rotatable_minimized} (appendix) show boxplots of strain energy per rotatable bonds, for the generated and minimized poses respectively, of molecules generated by different method.

\begin{figure}[!t]
    \begin{center}
    \centerline{\includegraphics[width=\linewidth]{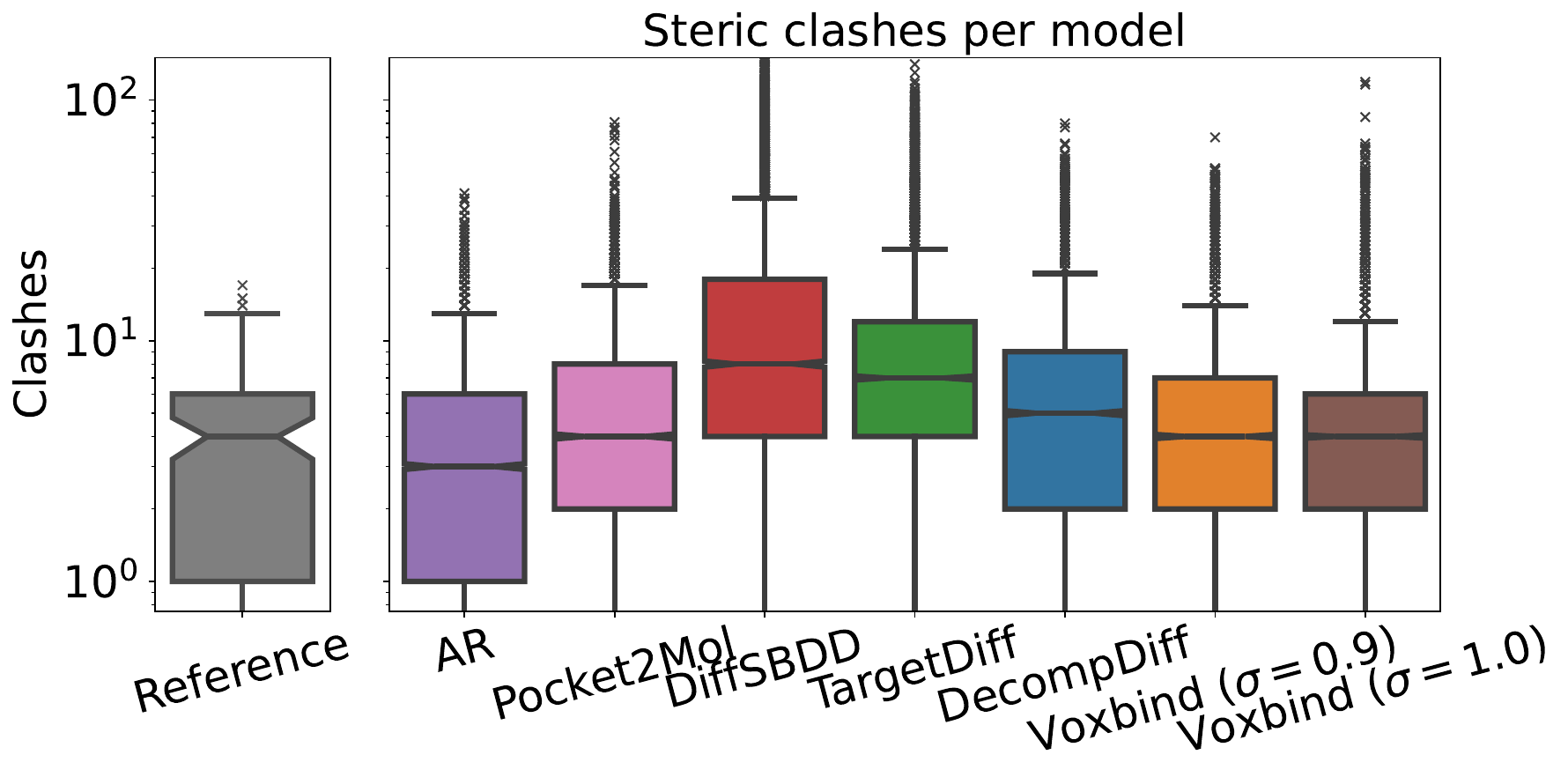}}
    \vskip -0.1in
    \caption{Number of steric clashes (lower is better) for the reference set and for molecules generated by each model.}
    \label{fig:clashes}
    \end{center}
    \vskip -0.2in
\end{figure}

\begin{figure*}[!t]
    \begin{center}
    \centerline{\includegraphics[width=1.\linewidth]{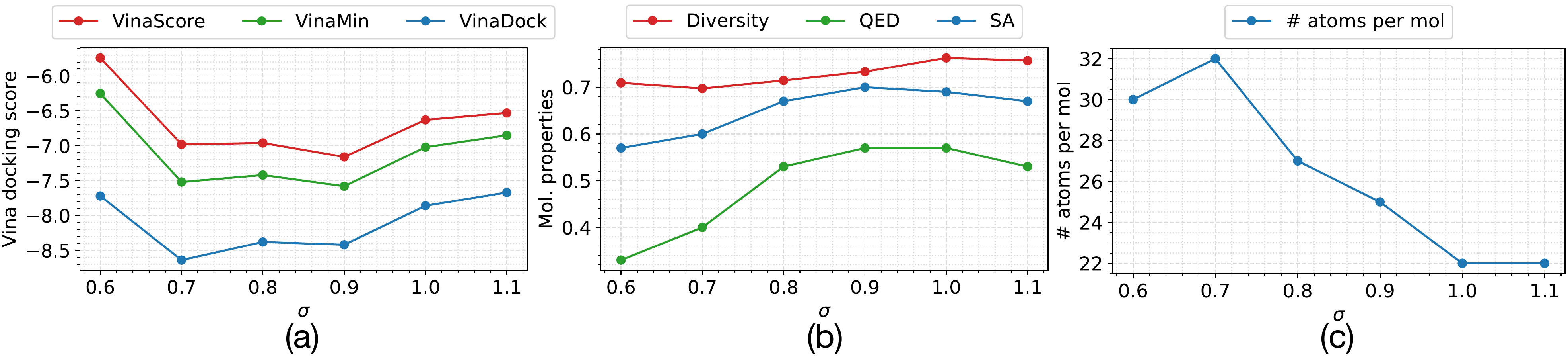}}
    \vskip -0.15in
    \caption{Effect of different noise level on (a) Vina metrics (VinaScore, VinaMin, VinaDock, lower is better), (b) molecular properties metrics (QED, SA and diversity, higher is better) and (c) number of atoms per molecules. In this experiment, we generate 100 molecules for the 100 pockets on validation set.}
    \label{fig:ablation_sigma}
    \end{center}
\vskip -0.35in
\end{figure*}
Next, we measure the consistency of rigid fragments/sub-structures (\eg, all carbons in a benzene ring should be in the same plane). We measure such consistency the same way  as~\citet{guan20233tdiff}. First, we optmize each molecule with  Merck molecular force field (MMFF)~\citep{halgren1996merck}. Then, we break the molecules into non-rotable fragments and, for each fragment, compute the RMSD between atoms coordinates before and after optimization. \autoref{fig:energy_min_and_rigid_rmsd}(b) (appendix) shows the RMSD for rigid fragments of different sizes. \modelname~models generate \emph{rigid fragments that are more consistent} than those of other approaches.

Finally, we compare how models capture different atomic bond distances. 
\autoref{tab:bond_distances} (appendix) shows the Jensen-Shannon divergence (JSD) between models and the reference set, for different types of bonds. We observe that \modelname~achieves lower JSD in most bond types, particularly double bonds and aromatic bonds. 
Moreover, we compare how models capture different atomic bond distances. 
Finally, \autoref{fig:interactions} shows the frequency for the following four types of protein-ligand interactions considered by PoseCheck.

\vspace{-.1in}\paragraph{Steric clashes.}
\autoref{fig:clashes} shows the number of steric clashes per ligand (on their generated poses). Our models generate ligands \emph{with less clashes} than other methods (with the exception of AR, which perform poorly in all other metrics).
In fact, \modelname$_{\sigma=0.9}$ and \modelname$_{\sigma=1.0}$ have mean clash score of $5.1$ and $5.3$, while AR, Pocket2Mol, DiffSBD, TargetDiff and DecompDiff have mean of $4.2$, $5.8$, $15.4$, $10.8$ and $7.1$ clashes, respectively. The reference set, for comparison, has mean clash score of $4.7$.

\vspace{-.1in}\paragraph{Sampling time.}  
\modelname~is \emph{more efficient} at sampling than other methods, sometimes an order of magnitude faster.
\modelname~takes, on average, 492.2 seconds to generate 100 valid molecules with one A100 GPU, while Pocket2Mol, TargetDiff and DecompDiff requires 2,544, 3,428 and 6,189 seconds respectively (baselines running times taken from~\citet{guan23ddiff}).

\vspace{-.15in}\paragraph{Qualitative results.} \autoref{fig:qualitative_res}, and \autoref{fig:qualitative_res_extra}, \autoref{fig:qualitative_res_extra2} on appendix show examples of generated ligands from our model.

\subsection{Ablation studies}
\paragraph{Noise level}\label{sec:ablation_noise_level}
The noise level is an important hyperparameter of our model. 
As the noise level increases, it becomes easier to sample from the smoothed distribution. However, denoising becomes harder. To find the best empirical noise level, we train our model with different noise levels (all other hyperparameters are kept the same). Then, we compare the quality of the samples conditioned on pockets from the \emph{validation} set. \autoref{fig:ablation_sigma} shows how different metrics change as the noise level changes. We found that $\sigma=.9$ and $\sigma=1.0$ achieve the best results on the validation set and chose to report results with those two levels. 

\vspace{-.1in}\paragraph{Data augmentation}\label{sec:ablation_data_aug}
We also train a version of \modelname~without data augmentation. In terms of molecular properties (QED and SA), the performance is very similar. However, we see a big difference in terms of binding affinity metrics. 
\autoref{fig:ablation_aug} (appendix) 
shows the median VinaScore of \modelname~trained with and without data augmentation for each target (we use $\sigma=0.9$ in this experiment). The model trained with data augmentation has better affinity score on the generated poses on 85\% of the targets.

\section{Conclusion}
This paper presents \modelname, a new score-based generative model for SBDD. We extended the neural empirical Bayes formalism and the walk-jump sampling algorithm to the conditional setting (cWJS) and show that our model outperforms previous work on an extensive number of computational metrics on a popular benchmark, while being faster to generate samples.

Our approach is flexible and allows us to adapt sampling to many practical SBDD settings, without any additional training. For example, if we start with a ligand that binds to the pocket of interest, we can initialize the MCMC chain with a smooth version of that ligand instead of noise.\footnote{We can easily adapt the sampling to other applications, \eg, scaffold hopping or linking, by initializing the chains with molecular sub-parts fragments. These are related to ``in-painting'' tasks in computer vision.}

The flexibility and expressivity of 3D U-Nets comes at the cost of increased memory consumption.
Therefore, the volume in 3D space that we can process is bounded by GPU memory. As shown empirically, our approach works well for drug-like molecule generation. However, more work (\eg, on data representation and architecture) needs to be done to scale generation to larger molecules like nucleic acids and proteins. 
Additional future work includes better modeling of synthetic accessibility or integrating pocket dynamics into the generation process.

\vspace{-.15in}\paragraph{Broader Impact.} 
Structure-based drug design is an important component in modern drug discovery research and development. This is a very long and challenging endeavor that involves many steps.
In this paper we propose a new pocket-conditional generative model, which deals with one of these steps.
There is still a lot of work that need to be done to validate these kinds of models in practice (\eg, wet-lab experimental validation, clinical trials, etc). That been said, if successful, advances in this field can directly impact quality of human health.
Like many other modern powerful technologies, we need to ensure that these models are deployed in ways that are safe, ethical, accountable and exclusively beneficial to society. 

\vspace{-.15in}\paragraph{Acknowledgements.} 
The authors are grateful to Charlie Harris for assistance in running the Posecheck benchmark.
We also would like to thank the Prescient Design team for helpful discussions and Genentech’s HPC team for providing a reliable environment to train/analyse models.

\bibliography{main}
\bibliographystyle{icml2024}

\newpage
\appendix
\onecolumn
\section{Additional implementation details}
\subsection{Voxelized molecules}\label{sec:appendix_vox_mols}
We represent molecules as voxelized atomic densities. We follow the same approach as~\citep{pinheiro2023voxmol}, as it has been shown to work well in practice for drug-like molecule generation. We apply the same approach for ligands and protein pockets. First, we convert each atom (of each molecule) into 3D Gaussian-like densities:
\begin{equation}
V_a(d, r_a) = \text{exp}\Big(-
\frac{d^2}{(.93\cdot r_a)^2} 
\Big),
\label{eq:atomic_density}
\end{equation}
where $V_a$ is defined as the fraction of occupied volume by atom $a$ of radius $r_a$ at distance $d$ from its center. We use radius $r_a=.5$ for all atom on ligand molecules (as in \citep{pinheiro2023voxmol}) and use the Van der Waals radii for pocket atoms.
Then, we compute the occupancy of each voxel in the grid:
\begin{equation}
\text{Occ}_{i,j,k} = 1 - \prod_{n=1}^{N_a} \big(1 - V_{a_n}(\Vert C_{i,j,k} - x_n \Vert, r_{a_n} )\big),
\end{equation}
where $\{a_n\}_{n=1}^{N_a}$ are the atoms of the molecule, $C_{i,j,k}$ are the coordinates (i,j,k) in the grid and $x_n$ is the coordinates of the center of atom $n$ \cite{li2014modeling}. Each atom type occupy a different channel grid (similar to R,G,B channels of images) and they take values between 0 and 1.
We use use a grid with $64^3$ voxels with resolution of $.25$\AA~per voxel. We model hydrogens implicitly and consider seven chemical elements for ligands (C, O, N, S, F, Cl and P) and four for pockets (C, O, N, S). This results in voxel grids with dimensions $d_x=7\times64\times64\times64$ and $d_z=4\times64\times64\times64$ for ligand and protein protein pockets, respectively.

Every ligand and its pocket pair are centered around the center of the mass of the ligand. During training, each training sample is augmented random augmentation (applied to both ligand and pocket).
These augmentations are made of random translation (uniform value between [-1,1] on 3D coordinates) and rotation (uniform value between [0, 2$\pi$) on three Euler angles). These augmentations are applied on the point cloud before voxelizing them.
We use the python package PyUUL~\cite{orlando2022pyuul} to generate the voxel grids from the raw molecules (\texttt{.sdf} or \texttt{.pdb} format).

\subsection{Sampling}\label{sec:appendix_sampling}
The walk-jump sampling approach is very flexible and allows us to configure sampling in different ways. For example, we can chose the number of walk steps between jumps, the maximum number of walk steps per chain or the number of chains run in parallel. Different sampling hyperparameters can change the statistics of samples, \eg, we increase sampling speed and reduce diversity if we reduce the number of walk steps between jumps.

Therefore, we decided to fix a set of sampling hyperparameters for benchmark purposes. In all our experiments we generate samples in the following way (100 chains in parallel): (i) initialize a chain $y_0$ (from noise and a pocket) and walk 400 Langevin MCMC \emph{warm-up} steps to get to $y_{400}$, (ii) create a batch (size 100) with copies of the tensor $y_{400}$ (iii) \emph{walk} 100 steps then \emph{jump} to estimate a clean molecule at step $\hat{x}_{500}$, (iv) return to step (i) and repeat until generate 100 valid samples.  
When sampling conditioned on pocket and ligand, we take $50$ warm-up steps and $50$ walk steps. We found this setting to provide a good trade-off in terms of performance/speed on validation set. However, it is by no means optimal and results can possibly be further improved by finding a better sampling recipe. 

We extract the atomic coordinates from the generated voxelized grids following the same approach as \citep{pinheiro2023voxmol}: we set all voxel values less than $0.1$ to $0$ then run a peak finding algorithm to get the 3D coordinates of each atom. The identity of the atom is the channel in the voxel grid. 
Once we have the set of atoms (types and coordinates), we follow previous work and use OpenBabel~\citep{o2011obabel} to assign bonds to the atoms.

The qualitative samples in this paper are sampled with 400 warm-up step followed by 100 steps for 10,000 steps on a \emph{single} MCMC chain.

\newpage
\section{Additional results}
\begin{figure}[!htb]
\begin{center}
\centerline{\includegraphics[width=1.0\linewidth]{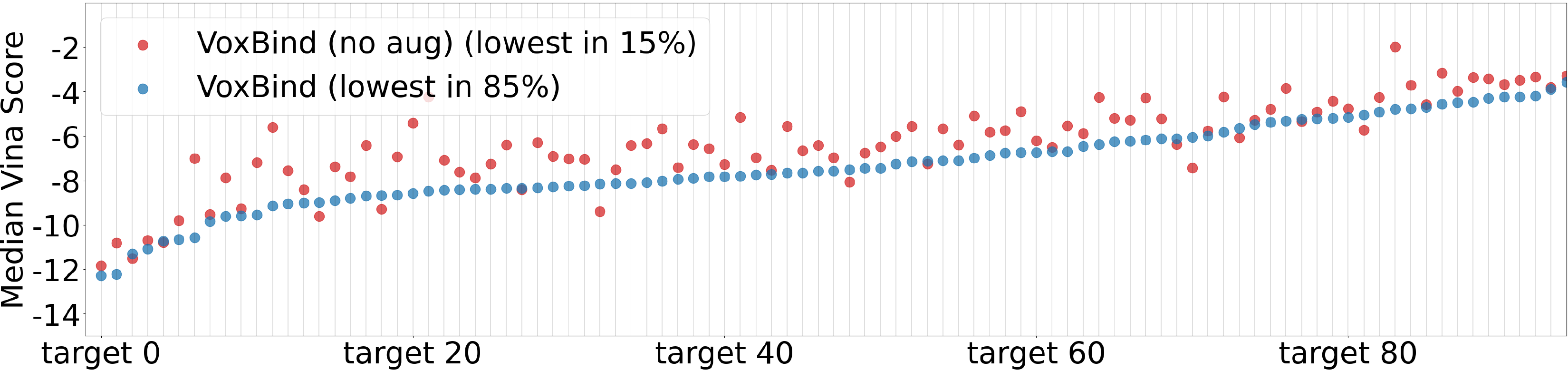}}
\vskip -0.15in
\caption{Median VinaScore of all generated molecules for each target on the test set (sorted by \modelname's score). We compare \modelname~trained with and without data augmentation with $\sigma=0.9$.}
\label{fig:ablation_aug}
\end{center}
\vskip -0.3in
\end{figure}

\begin{figure}[!h]
\vskip 0.2in
\begin{center}
\centerline{\includegraphics[width=1.\linewidth]{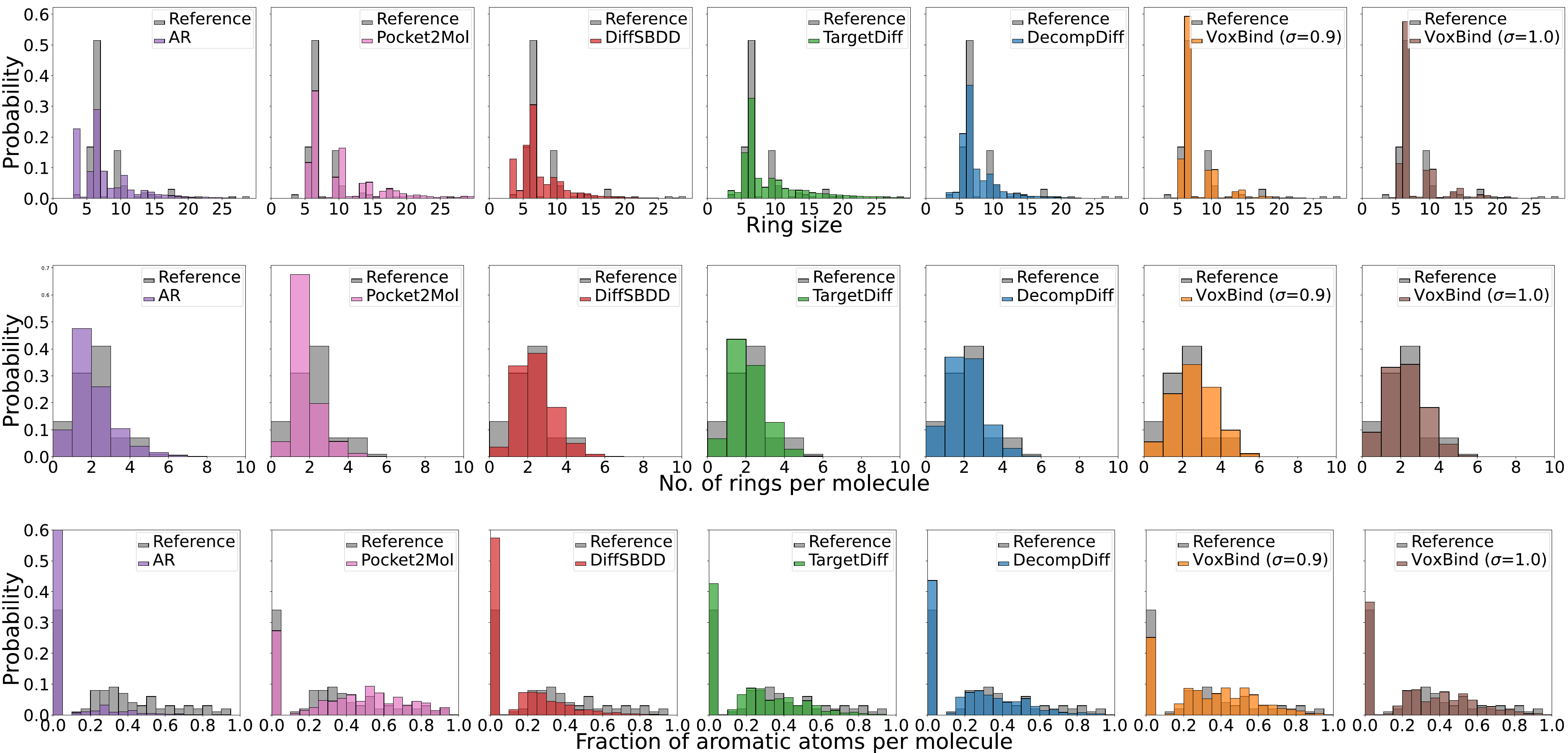
}}
\caption{Histograms showing (top) distributions ring sizes for all rings, (mid) distributions of number of rings per molecule and (bottom)  distributions of fraction of aromatic atoms per molecule in all generated molecules from different methods. The underlying distribution from the reference is also shown in each plot.}
\label{fig:histogram_rings}
\end{center}
\vskip -0.2in
\end{figure}

\begin{figure}[!htb]
\vskip 0.2in
\begin{center}
\centerline{\includegraphics[width=1.\linewidth]{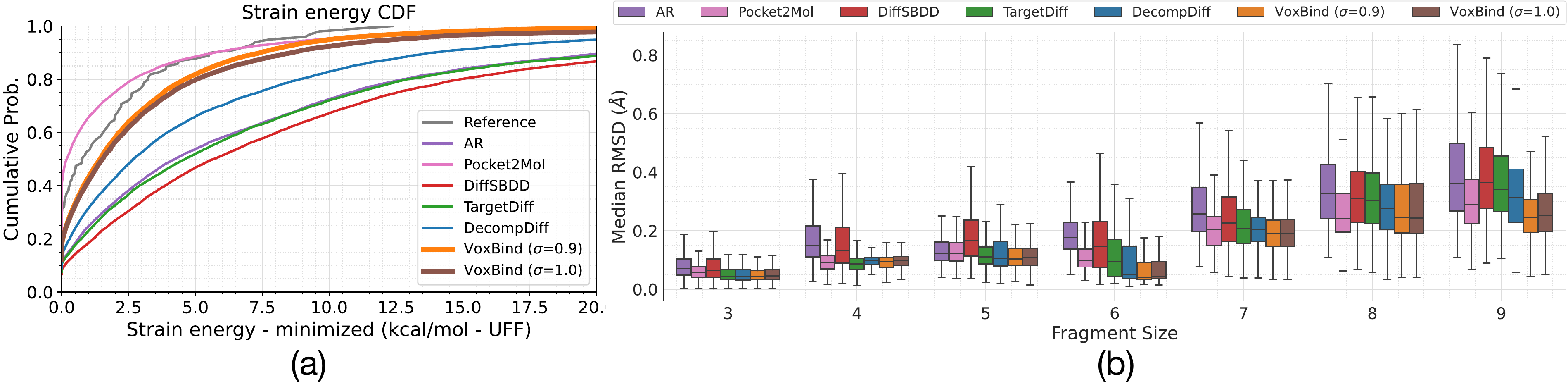}}
\vskip -0.15in
\caption{(a) The cumulative distribution function of strain energy of molecules after small force-field relaxation. As expected, the strain energies in this setting are much lower than with the raw generated poses. (b) Median RMSD ($\downarrow$) between fragments of generated molecules before and after force-field optimization.}
\label{fig:energy_min_and_rigid_rmsd}
\end{center}
\vskip -0.2in
\end{figure}

\begin{figure}[!htb]
    \begin{center}
    \centerline{\includegraphics[width=.75\linewidth]{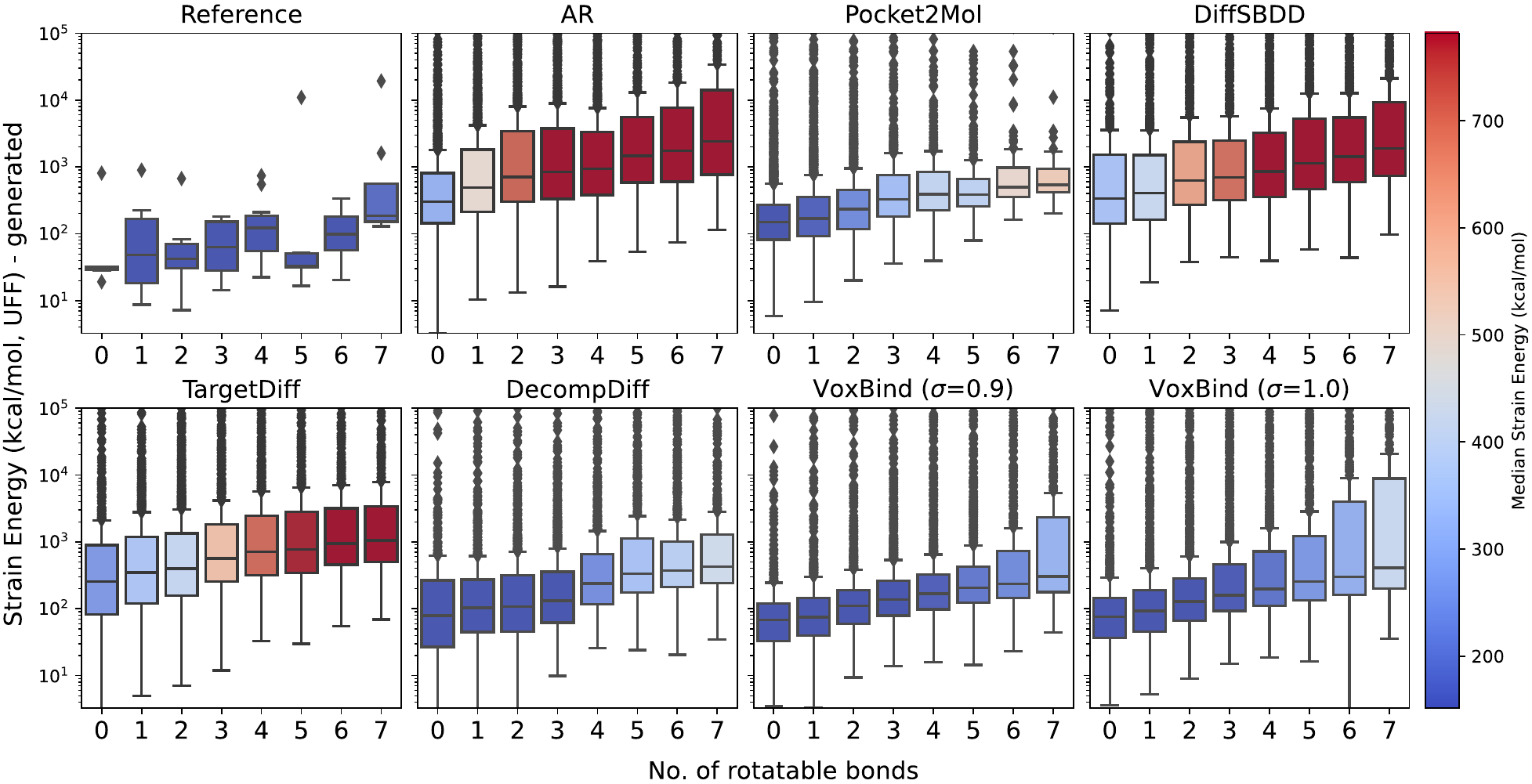}}
    \vskip -0.15in
    \caption{Boxplots of strain energies (lower is better) of generated molecules (on their generated poses) per number of rotatable bonds for all methods. Box color shows median strain value.}
    \label{fig:energy_by_rotatable_generated}
    \end{center}
    \vskip -0.2in
\end{figure}
\begin{figure}[!htb]
    \begin{center}
    \centerline{\includegraphics[width=.75\linewidth]{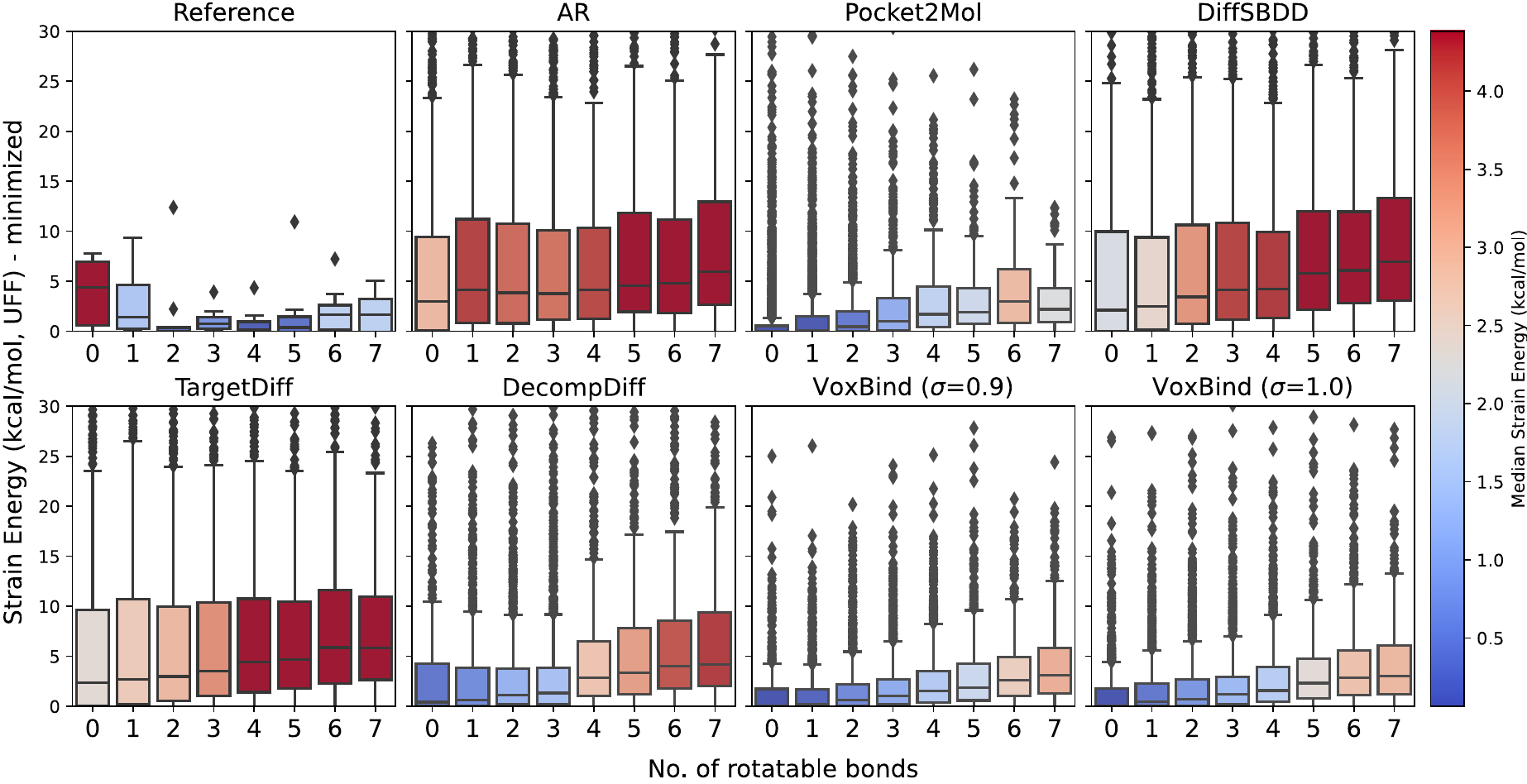}}
    \vskip -0.15in
    \caption{ Boxplots of strain energies (lower is better) of generated molecules (after local minimization) per number of rotatable bonds for all methods. Box color shows median strain value.}
    \label{fig:energy_by_rotatable_minimized}
    \end{center}
    \vskip -0.2in
\end{figure}

\setchemfig{atom sep=20pt}
\begin{table}[!htb]
  \caption{Jensen-Shannon divergence ($\downarrow$) between reference and generated molecules' distributions of bond distances. The symbols ``-- '', ``='', ``:'' represent single bond, double bond and aromatic bond, respectively. For each metric, we {\bf bold} and \underline{underline} the best and second best methods, respectively. \label{tab:bond_distances}}
  \begin{center}
  \resizebox{.8\textwidth}{!}{%
  \begin{tabular}{l|ccccccccc}
    & \chemfig{C-C} & \chemfig{C=C} &  \chemfig{C-N} & \chemfig{C=N} & \chemfig{C-O} & \chemfig{C=O} & \chemfig{C:C} & \chemfig{C:N} \\
    \shline
    AR         & .609 & .620 & .474 & .635 & .492 & .558 & .451 & .552 \\
    Pocket2Mol & .496 & .561 & .416 & .629 & .454 & .516 & .416 & .487 \\
    TargetDiff & .369 & \textbf{.505} & .363 & .550 & .421 & .461 & .263 & .235 \\
    DecompDiff$^\dagger$ & \underline{.359} & .537 & \textbf{.344} & .584 & \underline{.376} & .374 & .251 & .269 \\
    \hline
    \modelname$_{\sigma=0.9}$ & .372 & \underline{.528} & \underline{.351} & \underline{.528} & .400 & \textbf{.326} & \underline{.215} & \textbf{.186} \\
    \modelname$_{\sigma=1.0}$ & \textbf{.357} & .533 & .354 & \textbf{.418} & \textbf{.354} & \underline{.335} & \textbf{.210} & \underline{.191} \\
    
  \end{tabular}
  }
  \end{center}
\end{table}

\begin{figure}[!ht]
\vskip 0.2in
\begin{center}
\centerline{\includegraphics[width=.75\linewidth]{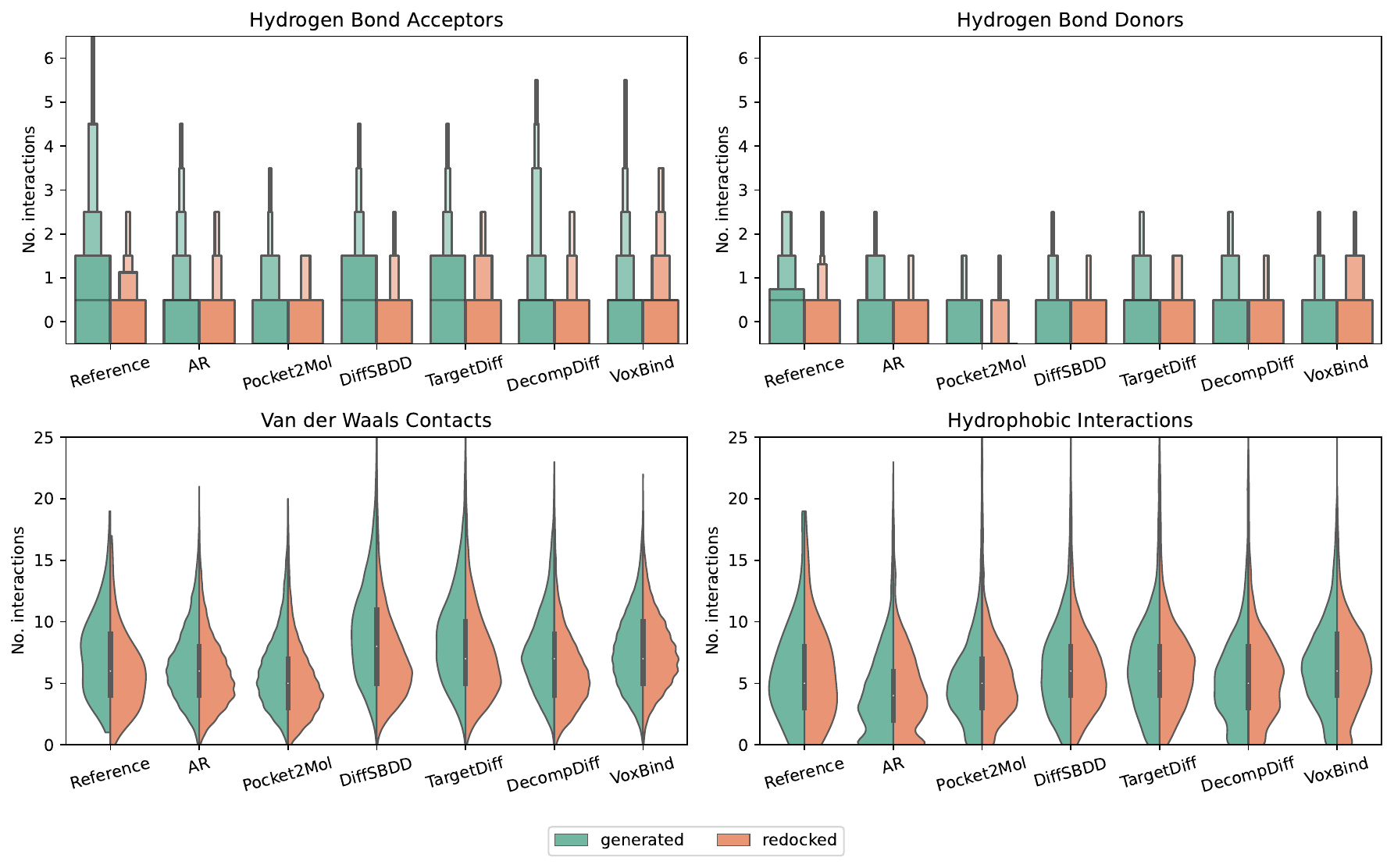
}}
\caption{ Protein-ligand interactions in generated poses (green) and redocked poses (orange). The frequency of (a) hydrogen bond acceptors, (b) hydrogen bond donors, (c) Van der Waals contacts and (d) hydrophobic interactions are shown.}
\label{fig:interactions}
\end{center}
\vskip -0.2in
\end{figure}

\begin{figure}[h]
\vskip 0.2in
\begin{center}
\centerline{\includegraphics[width=1.\linewidth]{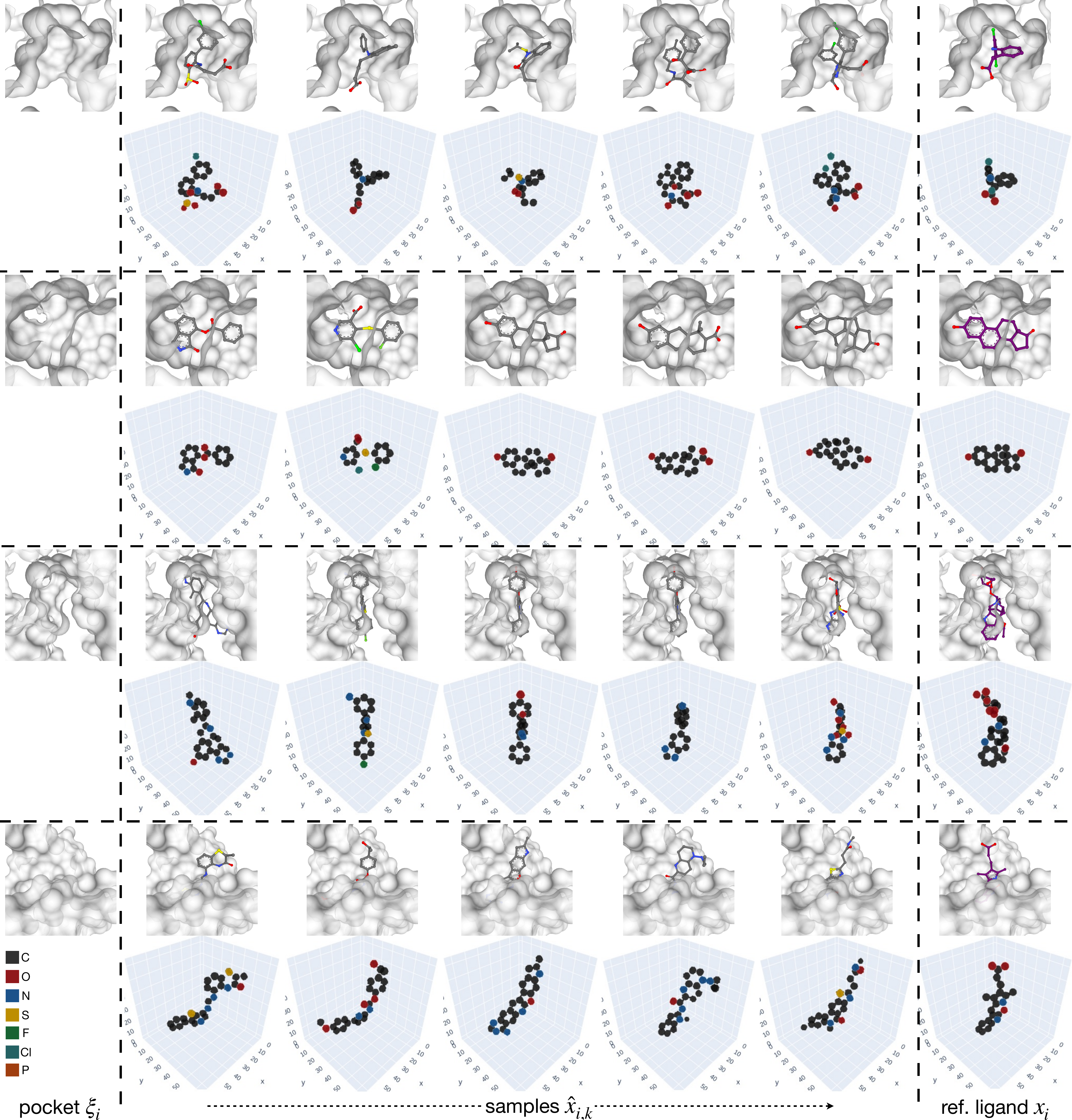}}\vskip -0.15in
\caption{Example of generated ligands $\hat{x}_{i,k}$ given pocket $\pocket_i$. Each row represents a single chain of samples for a given protein pocket (\texttt{2i2z}, \texttt{1ogx}, \texttt{3u57}, \texttt{4jlc} from top to bottom). The samples from each row are generated from the same MCMC chain. The provided ground-truth ligands are shown on the last column.}
\label{fig:qualitative_res_extra}
\end{center}
\vskip -0.2in
\end{figure}

\begin{figure}[hb]
\vskip 0.2in
\begin{center}
\centerline{\includegraphics[width=1.\linewidth]{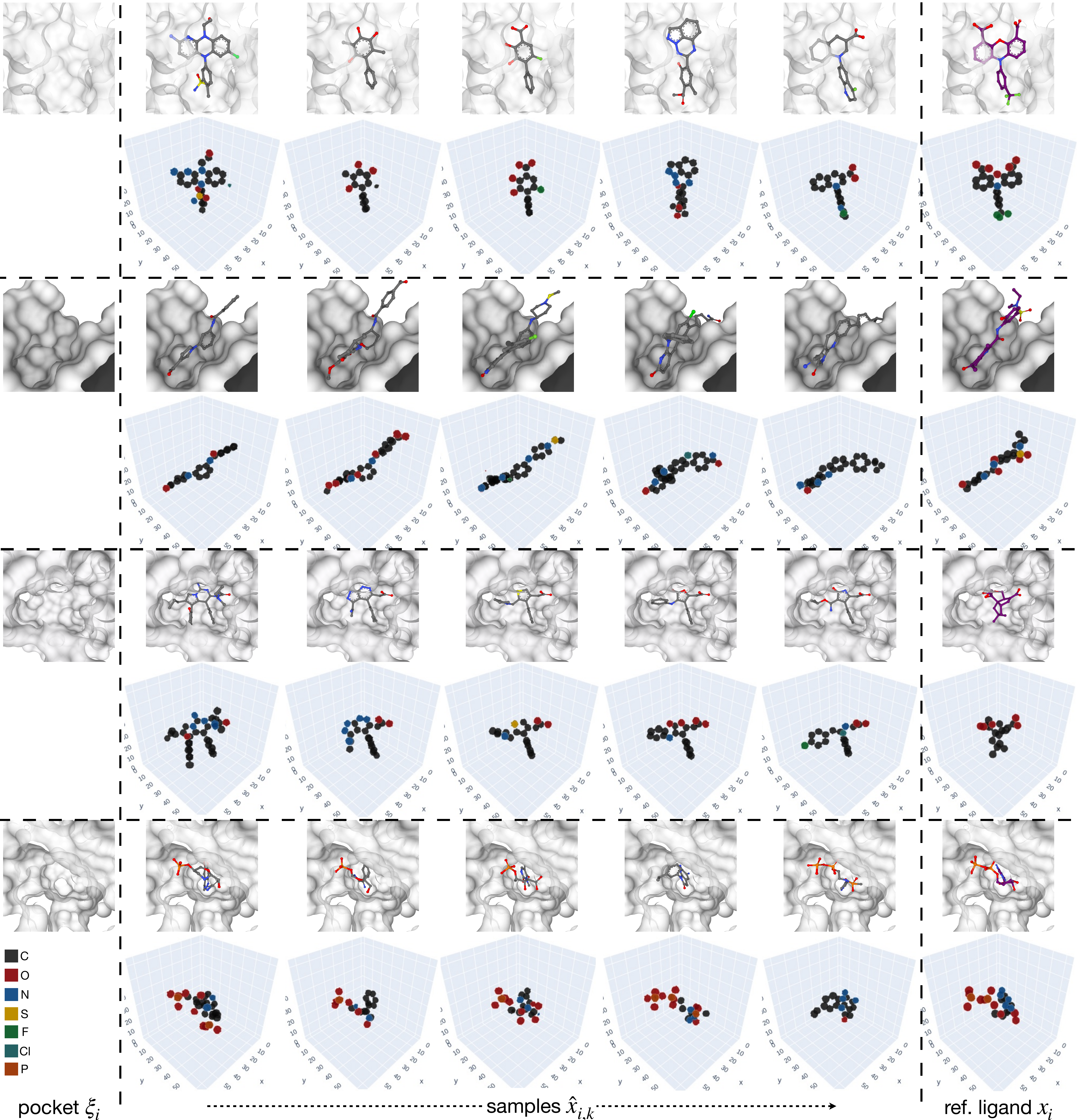}}\vskip -0.15in
\caption{Example of generated ligands $\hat{x}_{i,k}$ given pocket $\pocket_i$. Each row represents a single chain of samples for a given protein pocket (\texttt{5aks}, \texttt{5crz}, \texttt{5l1v}, \texttt{1e8h} from top to bottom). The samples from each row are generated from the same MCMC chain. The provided ground-truth ligands are shown on the last column.}
\label{fig:qualitative_res_extra2}
\end{center}
\vskip -0.2in
\end{figure} 

\end{document}